%% file: sample.tex
\documentclass[twoside,11pt]{article}

\usepackage{blindtext}

\usepackage[abbrvbib, preprint]{jmlr2e}

\usepackage{jmlr2e}

\usepackage{lastpage}
\jmlrheading{23}{2026}{1-\pageref{LastPage}}{1/21; Revised 5/22}{9/22}{21-0000}{Anirudh Satheesh and Vaneet Aggarwal}
\usepackage{amsmath}
\usepackage{amssymb}
\usepackage{mathtools}
\usepackage{bm}
\usepackage{booktabs}
\usepackage[textsize=tiny]{todonotes}
\usepackage{enumitem}
\usepackage{tablefootnote}
\usepackage{algorithmic}
\allowdisplaybreaks
\usepackage{algorithm}
\newtheorem{assumption}{Assumption}
\newcommand{\tablefootnotemark}[1]{\textsuperscript{\getrefnumber{#1}}}

\ShortHeadings{Regret Analysis of Unichain Average Reward Constrained MDPs}{Regret Analysis of Unichain Average Reward Constrained MDPs}
\firstpageno{1}

\makeatletter
\def\ps@jmlrtps{%
  \let\@oddhead\@empty
  \let\@evenhead\@empty
  \let\@oddfoot\@empty
  \let\@evenfoot\@empty
}
\makeatother

\begin{document}

\title{Regret Analysis of Unichain Average Reward Constrained MDPs with General Parameterization}

\author{\name Anirudh Satheesh \email anirudhs@terpmail.umd.edu \\
       \addr 
       University of Maryland\\
       College Park, MD 20742, USA
       \AND
       \name Vaneet Aggarwal \email vaneet@purdue.edu \\
       \addr 
       Purdue University\\
       West Lafayette, IN 47907, USA}

\editor{My editor}

\maketitle

\begin{abstract}
We study infinite-horizon average-reward constrained Markov decision processes (CMDPs) under the unichain assumption and general policy parameterizations. Existing regret analyses for constrained reinforcement learning largely rely on ergodicity or strong mixing-time assumptions, which fail to hold in the presence of transient states. We propose a primal--dual natural actor--critic algorithm that leverages multi-level Monte Carlo (MLMC) estimators and an explicit burn-in mechanism to handle unichain dynamics without requiring mixing-time oracles. Our analysis establishes finite-time regret and cumulative constraint violation bounds that scale as $\tilde{O}(\sqrt{T})$, up to approximation errors arising from policy and critic parameterization, thereby extending order-optimal guarantees to a significantly broader class of CMDPs.
\end{abstract}

\input{Sections/introduction}
\input{Sections/formulation}
\input{Sections/algorithm}
\input{Sections/assumptions}
\input{Sections/regret_analysis}

\input{Sections/conclusion}





\bibliography{sample}
\appendix
\input{Sections/Appendix/related_work}

\input{Sections/Appendix/algorithm}
\input{Sections/Appendix/stochastic_linear_recursion}
\input{Sections/Appendix/supporting_lemmas}
\input{Sections/Appendix/critic_analysis}
\input{Sections/Appendix/npg_analysis}
\input{Sections/Appendix/cmdp_analysis}
\input{Sections/Appendix/auxilliary_lemmas}
\input{Sections/Appendix/limitations}

\end{document}

%% file: Sections/introduction.tex
\section{Introduction}

Reinforcement learning (RL) with constraints is fundamental in applications where agents must optimize performance while satisfying safety, resource, or regulatory requirements, including autonomous driving \citep{wen2020safe}, power grid control \citep{wu2024chance}, and fair recommendation systems \citep{cai2023rec}. These problems are naturally modeled as infinite-horizon average-reward constrained Markov decision processes (CMDPs), which seek to maximize long-run average reward subject to average cost constraints \citep{chen2022learning}.

Most existing CMDP guarantees rely on ergodicity assumptions requiring all states to be recurrent and the Markov chain to mix rapidly, as shown in Table~\ref{table: related work} \citep{ghosh2023achieving, bai2024learning, xu2025global}. Such assumptions are violated in many real-world systems that contain transient states, initialization or terminal phases, or periodic behavior. While some works relax ergodicity, they are largely restricted to tabular parameterizations and do not extend to high-dimensional settings with general function approximation \citep{jaksch2010near,chen2022learning, wei2022provably, altman2021constrained, agarwal2022regret,agarwal2022concave,agrawal2024optimistic}. This raises the following question: can we achieve order-optimal regret for constrained RL in the more general unichain setting, which permits transient states and periodicity, while using general policy parameterizations?

\subsection{Challenges and Contributions}

Extending constrained RL to the unichain setting presents two core challenges. First, existing unichain CMDP analyses are limited to tabular methods, leaving open whether order-optimal guarantees hold under neural or other rich parameterizations. Second, unlike unconstrained RL, CMDPs require coordinated primal-dual updates. Naive adaptations of unconstrained unichain techniques lead to suboptimal regret due to the slower learning rates required for constraint satisfaction. We resolve these challenges via two key innovations, yielding the first order-optimal regret guarantees for unichain CMDPs with general parameterization.

\paragraph{MLMC for Unichain Markov Chains.}
We develop the first multi-level Monte Carlo (MLMC) estimators for unichain Markov chains. Unlike prior MLMC analyses that rely on exponential mixing, our approach extends the Poisson equation decomposition to jointly handle transient states and the recurrent class. Using the hitting time to the recurrent class $C_{\text{hit}}$ and the mixing constant within the recurrent class $C_{\text{tar}}$, we establish bias and variance bounds of $\tilde{O}(C_{\text{tar}}^2/T_{\max})$ and $\tilde{O}(C_{\text{tar}}^2 \log T_{\max})$ respectively (Lemmas~\ref{lemma: mse unichain markov chain} and~\ref{lemma: mlmc statements}). This achieves the accuracy of $T_{\max}$ samples using only $O(\log T_{\max})$ samples in expectation, without requiring knowledge of mixing parameters.

\paragraph{Logarithmic Burn-in.}
To ensure that critic and policy gradient updates originate from the recurrent class, prior unichain RL analyses employ burn-in periods of length $O(\sqrt{T})$ \citep{ganesh2025regret}, leading to prohibitive total burn-in cost when applied to CMDPs. We show that a burn-in of length $O(\log T)$ suffices by establishing polynomial decay $2^{-B/(2C_{\text{hit}})}$, which is sufficient when combined with our MLMC estimators and $K = O(T/\log T)$ epochs.

We summarize our contributions below:
\begin{itemize}
\item We establish the first $\tilde{O}(\sqrt{T})$ regret bound for unichain CMDPs with general policy and critic parameterizations (Theorem~\ref{theorem: main_regret}), extending order-optimal guarantees beyond both tabular methods and ergodic assumptions. See Table \ref{table: related work} for detailed comparison. 
\item We introduce the first MLMC estimators for unichain Markov chains, reducing inner loop per-iteration sample complexity from $O(\sqrt{T})$ to $O(\log T)$ while explicitly accounting for transient and recurrent dynamics.
\end{itemize}

\begin{table}[t]
\centering
\caption{Comparison of relevant algorithms for infinite-horizon average-reward Reinforcement Learning in the \textbf{model-free} setting. Our work is the first to achieve order-optimal regret and constraint violation for \textbf{Constrained} MDPs with \textbf{General} parameterization under the \textbf{Unichain} assumption. We provide a more comprehensive comparison to prior work in Section~\ref{sec: related work}.}
\resizebox{\textwidth}{!}{%
\begin{tabular}{@{}lccccc@{}}
\toprule
\textbf{Algorithm} & \textbf{Param.} & \textbf{Ergodic?} & \textbf{Constr.?} & \textbf{Regret} & \textbf{Violation} \\ \midrule
\citet{agarwal2022regret,agarwal2022concave} & Tabular\tablefootnote{This work is model-based, and included for comparison.} & Yes & Yes & $\tilde{\mathcal{O}}(\sqrt{T})$ & 0 \\
\citet{ghosh2023achieving} & Linear & Yes & Yes & $\tilde{\mathcal{O}}(\sqrt{T})$ & $\tilde{\mathcal{O}}(\sqrt{T})$ \\
\citet{bai2024learning} & General & Yes & Yes & $\tilde{\mathcal{O}}(T^{4/5})$ & $\tilde{\mathcal{O}}(T^{4/5})$ \\
\citet{xu2025global} & General & Yes & Yes & $\tilde{\mathcal{O}}(\sqrt{T})$ & $\tilde{\mathcal{O}}(\sqrt{T})$ \\
\citet{jaksch2010near}\tablefootnote{This work assumes an MDP with a finite diameter.} & Tabular & No & No & $\tilde{\mathcal{O}}(\sqrt{T})$ & N/A \\
\citet{agrawal2024optimistic}\tablefootnote{These works assume a weakly communicating MDP. \label{weakly communicating}} & Tabular & No & No & $\tilde{\mathcal{O}}(T^{2/3})$ & N/A \\
\citet{ganesh2025regret} & General & No & No & $\tilde{\mathcal{O}}(\sqrt{T})$ & N/A \\
\citet{chen2022learning}\tablefootnotemark{weakly communicating} & Tabular & No & Yes & $\tilde{\mathcal{O}}(\sqrt{T})$ & $\tilde{\mathcal{O}}(\sqrt{T})$ \\
\citet{wei2022provably}\tablefootnotemark{weakly communicating} & Tabular & No & Yes & $\tilde{\mathcal{O}}(T^{5/6})$ & 0 \\
\textbf{Ours} & \textbf{General} & \textbf{No} & \textbf{Yes} & $\tilde{\mathcal{O}}(\sqrt{T})$ & $\tilde{\mathcal{O}}(\sqrt{T})$ \\ \midrule
\textbf{Lower Bound} & -- & -- & -- & $\Omega(\sqrt{T})$ & -- \\ \bottomrule
\end{tabular}%
}
\label{table: related work}
\end{table}

%% file: Sections/formulation.tex
\section{Formulation}
In this work, we consider an infinite-horizon average-reward Constrained Markov Decision Process (CMDP) denoted as $\mathcal{M} = (\mathcal{S}, \mathcal{A}, r, c, P, \rho)$, where $\mathcal{S}$ is the state space, $\mathcal{A}$ is the action space, $r : \mathcal{S} \times \mathcal{A} \to [0, 1]$ represents the reward function, $c : \mathcal{S} \times \mathcal{A} \to [-1, 1]$ is the constraint cost function, $P : \mathcal{S} \times \mathcal{A} \to \Delta(\mathcal{S})$ is the state transition function, where $\Delta(\mathcal{S})$ denotes the probability simplex over $\mathcal{S}$, and $\rho \in \Delta(\mathcal{S})$ indicates the initial distribution of states. A policy $\pi : \mathcal{S} \to \Delta(\mathcal{A})$ maps a state to an action distribution. Given a policy $\pi$, the policy induces a transition function $P^\pi : \mathcal{S} \to \Delta(\mathcal{S})$, given as:
\[
    P^\pi(s, s') = \sum_{a \in \mathcal{A}} P(s'|s, a)\pi(a|s), \quad \forall s, s' \in \mathcal{S}
\]

\begin{assumption}[Unichain MDP]
    \label{assumption: unichain}
    The CMDP $\mathcal{M}$ is such that, for every policy $\pi \in \Pi$, the induced Markov chain has a single recurrent class.
\end{assumption} 
This assumption does not require irreducibility (as transient states may be present) nor does it impose aperiodicity. Consequently, it is strictly weaker than the standard ergodicity assumptions commonly used in the literature. Under Assumption 1, the stationary distribution $d^\pi \in \Delta(|\mathcal{S}|)$ is well-defined, independent of the initial distribution $\rho$, and satisfies the balance equation: $(P^\pi)^\top d^\pi = d^\pi$.
\begin{definition}[Stationary Distribution]
    \label{definition: stationary distribution}
    Let $d^\pi \in \Delta(|\mathcal{S}|)$ denote the stationary distribution of the Markov chain induced by $\pi$, given by:
    \[
        d^\pi(s) = \lim_{T \to \infty} \frac{1}{T} \sum_{t=0}^{T-1} \Pr(s_t = s \mid s_0 \sim \rho, \pi)
    \]
\end{definition}
The average reward and average constraint cost of a policy $\pi$, denoted as $J^\pi_r$ and $J^\pi_c$ respectively, are defined as:
\begin{align}
    \label{eqn: average objective}
    J^\pi_g \triangleq \lim_{T \to \infty} \frac{1}{T} \mathbb{E}^\pi \left[ \sum_{t=0}^{T-1} g(s_t, a_t) \Bigg| s_0 \sim \rho \right]\!, g \in \{r, c\}
\end{align}
The expectations are computed over the distribution of all $\pi$-induced trajectories $\{(s_t, a_t)\}_{t=0}^\infty$, where $a_t \sim \pi(\cdot|s_t)$ and $s_{t+1} \sim P(\cdot|s_t, a_t)$, $\forall t \in \{0, 1, \ldots\}$. Under Assumption 1, the average reward is independent of the initial distribution $\rho$ and can be expressed as:
\begin{align}
    J^\pi_g = \mathbb{E}_{s \sim d^\pi, a \sim \pi(\cdot|s)}[g(s,a)] = (d^\pi)^\top r^\pi
\end{align}
where $r^\pi(s) \triangleq \sum_{a \in \mathcal{A}} g(s,a)\pi(a|s), \forall s \in \mathcal{S}$. The goal is to maximize the average reward while ensuring that the average cost exceeds a given threshold. Without loss of generality, we formulate this as:
\begin{align}
    \max_\pi J^\pi_r \quad \text{s.t.} \quad J^\pi_c \geq 0
\end{align}
When the state space $\mathcal{S}$ is large, this problem becomes difficult to solve because the search space of the policy $\pi$ increases exponentially with $|\mathcal{S}| \times |\mathcal{A}|$. We therefore consider a class of parameterized policies $\{\pi_\theta \mid \theta \in \Theta\}$ that indexes policies by a $d$-dimensional parameter $\theta \in \mathbb{R}^d$, where $d \ll |\mathcal{S}||\mathcal{A}|$. The original problem can then be reformulated as:
\begin{align}
    \label{eqn: CMDP}
    \max_{\theta \in \Theta} J^{\pi_\theta}_r \quad \text{s.t.} \quad J^{\pi_\theta}_c \geq 0
\end{align}
For notational simplicity, we use $J^{\pi_\theta}_g = J_g(\theta)$ for $g \in \{r, c\}$ throughout.
\begin{assumption}[Slater Condition]
    \label{assumption: slater}
    There exists a $\delta \in (0, 1)$ and $\bar{\theta} \in \Theta$ such that $J_c(\bar{\theta}) \geq \delta$.
\end{assumption}
This assumption ensures the existence of an interior point solution and is commonly used in model-free average-reward CMDPs \citep{xu2025global, wei2022provably, bai2024learning}. For the unichain setting, convergence properties cannot be characterized using the mixing time from the ergodic setting, since transient states and policy-dependent recurrent classes prevent uniform mixing guarantees. Thus, we use the following constants defined in \citet{ganesh2025regret}:

Let $\mathcal{S}^\theta_R \subseteq \mathcal{S}$ denote the recurrent class under policy $\pi_\theta$, and let $T_\theta := \inf\{t \geq 0 : s_t \in \mathcal{S}^\theta_R\}$ denote the first hitting time of the recurrent class. We define:
\begin{align}
    C^\theta_{\mathrm{hit}} := \max_{s \in \mathcal{S}} \mathbb{E}^\theta_s[T_\theta]
\end{align}
as the worst-case expected time to enter the recurrent class when starting from any state $s \in \mathcal{S}$. Note that if there are no transient states, then $C^\theta_{\mathrm{hit}} = 0$. Similarly, for any $s, s' \in \mathcal{S}$, let $T_{s'} := \inf\{t \geq 0 : s_t = s'\}$ be the first hitting time to state $s'$. Then we define:
\begin{align}
    C^\theta_{\mathrm{tar}} := \sum_{s' \in \mathcal{S}} d^\pi(s') \mathbb{E}^\theta_s[T_{s'}]
\end{align}
as the expected time to reach a state drawn from $d^{\pi_\theta}$, starting from a fixed state $s \in \mathcal{S}^\theta_R$. We define $C_{\mathrm{hit}} := \sup_\theta C^\theta_{\mathrm{hit}}$, $C_{\mathrm{tar}} := \sup_\theta C^\theta_{\mathrm{tar}}$, and $C := C_{\mathrm{hit}} + C_{\mathrm{tar}}$.

For any fixed $\theta \in \Theta$, there exist functions $Q^{\pi_\theta}_g : \mathcal{S} \times \mathcal{A} \to \mathbb{R}$ for $g \in \{r, c\}$ such that the following Bellman equation is satisfied $\forall (s,a) \in \mathcal{S} \times \mathcal{A}$:
\begin{align}
    \label{eqn: Bellman Equation}
    Q^{\pi_\theta}_g(s,a) = g(s,a) - J_g(\theta) + \mathbb{E}_{s' \sim P(\cdot|s,a)}[V^{\pi_\theta}_g(s')]
\end{align}
where the state value function $V^{\pi_\theta}_g : \mathcal{S} \to \mathbb{R}$ is defined as:
\begin{align}
    V^{\pi_\theta}_g(s) = \mathbb{E}_{a \sim \pi_\theta(\cdot|s)}[Q^{\pi_\theta}_g(s,a)]
\end{align}
Note that if the Bellman equation is satisfied by $Q^{\pi_\theta}_g$, then it is also satisfied by $Q^{\pi_\theta}_g + c$ for any arbitrary constant $c$. We define the advantage function as:
\begin{align}
    A^{\pi_\theta}_g(s,a) = Q^{\pi_\theta}_g(s,a) - V^{\pi_\theta}_g(s), \,  g \in \{r, c\}, \forall (s,a)
\end{align}
To ensure uniqueness, we impose the normalization:
\begin{align}
    \sum_{s \in \mathcal{S}} d^{\pi_\theta}(s) V^{\pi_\theta}_g(s) = 0
\end{align}
Under Assumption~\ref{assumption: unichain}, the policy gradient at $\theta$ can be expressed as:
\begin{align}
    \nabla_\theta J_g(\theta) = \mathbb{E}[A^{\pi_\theta}_g(s,a) \nabla_\theta \log \pi_\theta(a|s)],\, g \in \{r, c\}
\end{align}
where the expectation is taken under $\nu^{\pi_\theta}(s,a) \triangleq d^{\pi_\theta}(s)\pi_\theta(a|s)$, which denotes the state-action occupancy measure. However, standard policy gradients can be prone to instability, so we use the Natural Policy Gradient (NPG), whose direction $\omega^*_{\theta,g}$ for $g \in \{r,c\}$ is defined as:
\begin{align}
    \omega^*_{g,\theta} \triangleq F_{\theta}^{-1} \nabla_\theta J_g(\theta)
\end{align}
where $F_{\theta} \in \mathbb{R}^{d \times d}$ is the Fisher information matrix, formally defined as:
\begin{align}
    F_{\theta} = \mathbb{E}_{(s,a) \sim \nu^{\pi_\theta}}[\nabla_\theta \log \pi_\theta(a|s) \otimes \nabla_\theta \log \pi_\theta(a|s)]
\end{align}
where $\otimes$ denotes the outer product. The NPG direction can also be expressed as the solution to the following strongly convex optimization problem:
\begin{align}
    \label{eqn: npg convex optimization}
    \min_{\omega \in \mathbb{R}^d} f_g(\theta, \omega) := \frac{1}{2}\omega^\top F_{\theta}\omega - \omega^\top \nabla_\theta J_g(\theta)
\end{align}
The gradient of $f_g(\theta, \cdot)$ is obtained as $\nabla_\omega f_g(\theta, \omega) = F_{\theta}\omega - \nabla_\theta J_g(\theta)$. Equivalently, the NPG direction solves:
\begin{align}
    \label{eqn: optimal npg direction}
    \omega^*_{\theta,g} = \arg\min_{\omega \in \mathbb{R}^d} L_{\nu^{\pi_\theta}}(\omega, \theta)
\end{align}
where the compatible function approximation error is defined as:
\begin{align}
    L_{\nu^{\pi_\theta}}(\omega, \theta) = \frac{1}{2}\mathbb{E}\left[A^{\pi_\theta}_\lambda(s,a)\!-\!\omega^\top \nabla_\theta \log \pi_\theta(a|s)\right]^2
\end{align}
where \(A_\lambda^{\pi_\theta} = A_r^{\pi_\theta} - \lambda A_c^{\pi_\theta}\) and the expectation is taken over the state-action occupancy measure. This formulation enables gradient-based iterative procedures to compute the NPG direction.

%% file: Sections/algorithm.tex
\section{Algorithm}

\subsection{Primal-Dual Framework}

Our approach to solving the constrained optimization problem in \eqref{eqn: CMDP} is built upon a saddle point formulation of the Lagrangian:
\begin{align}
    \label{eqn: saddle point formulation}
    \max_{\theta \in \Theta} \min_{\lambda \geq 0} \mathcal{L}(\theta, \lambda) \triangleq J_r(\theta) + \lambda J_c(\theta)
\end{align}
where $\mathcal{L}(\cdot, \cdot)$ represents the Lagrange function and $\lambda$ denotes the dual variable (Lagrange multiplier). Starting from an initial configuration $(\theta_0, \lambda_0 = 0)$, we iteratively update the policy parameters and dual variable through policy gradient steps. At each outer iteration $k \in \{0, \ldots, K-1\}$, the ideal updates would be:
\begin{align}
    \theta_{k+1} = \theta_k + \alpha F(\theta_k)^{-1}\nabla_\theta \mathcal{L}(\theta_k, \lambda_k) \label{eqn: primal update}\\
    \quad \lambda_{k+1} = \mathcal{P}_{\left[0, \frac{2}{\delta}\right]}[\lambda_k - \beta J_c(\theta_k)] \label{eqn: dual_update}
\end{align}
where $\alpha$ and $\beta$ denote the primal and dual learning rates respectively, $\delta$ is the Slater parameter from Assumption~\ref{assumption: slater}, and $\Pi_{\Lambda}$ projects onto the set $\Lambda$. 

The central challenge is that computing $\nabla_\theta \mathcal{L}(\theta_k, \lambda_k)$, $F(\theta_k)$, and $J_c(\theta_k)$ exactly requires knowledge of the transition dynamics $P$ and occupancy measure $\nu^{\pi_{\theta_k}}$, which are unavailable. We therefore work with approximate updates:
\begin{align}
\theta_{k+1} = \theta_k + \alpha \omega_k, \quad \lambda_{k+1} = \mathcal{P}_{\left[0, \frac{2}{\delta}\right]}[\lambda_k - \beta \eta^k_c] \label{eq:approx_update}
\end{align}
where $\omega_k$ estimates the natural policy gradient (NPG) direction $F(\theta_k)^{-1}\nabla_\theta \mathcal{L}(\theta_k, \lambda_k)$, and $\eta^k_c$ estimates $J_c(\theta_k)$. The remainder of this section details how we construct these estimates efficiently.

\subsection{Key Algorithmic Innovations}

Before presenting the detailed estimation procedures, we highlight two critical innovations that enable order-optimal regret in the unichain setting.

\subsubsection{Multi-Level Monte Carlo for Unichain MDPs}

To overcome the sample complexity bottleneck, we employ Multi-Level Monte Carlo (MLMC) estimation. The key insight is that MLMC achieves the bias of averaging $T_{\max}$ samples while requiring only $O(\log T_{\max})$ samples in expectation. This is accomplished through a telescoping structure that combines estimates across geometric levels.

At each inner iteration $h$, we sample a geometric random variable $Q^k_h \sim \text{Geom}(1/2)$ and construct trajectory length $l_{kh} = (2^{Q^k_h} - 1) \cdot \mathbb{I}(2^{Q^k_h} \leq T_{\max}) + 1$. The estimator combines a coarse baseline ($j=0$) with a telescoped correction at the finest affordable level. This structure dramatically reduces the effective batch size compared to prior work: while \citet{ganesh2025regret} requires batch sizes of $O(\sqrt{T})$ to achieve order-optimal regret in unconstrained unichain MDPs, our MLMC-based approach needs only $\mathcal{O}(\log T)$.

The theoretical foundation for this efficiency is established through two lemmas. First, we characterize the mean-squared error of time-averaged estimates under unichain dynamics:

\begin{lemma}[MSE of Markovian sample average for Unichain MDPs]
\label{lemma: mse unichain markov chain}
Let $(Z_t)_{t \ge 0}$ be a unichain Markov chain with stationary distribution $d_Z$. Let $g : \mathcal{Z} \to \mathbb{R}^d$ satisfy
\[
    \int g(z)\, d_Z(z) = 0, \quad \|g(z)\| \leq \sigma \quad \forall z \in \mathcal{Z},
\]
Let \(C_1 = C_{\mathrm{tar}}\) if \(Z_0 \in \mathcal{S}_R^\theta\) and \(C_1 = C\) otherwise. Then for all $N \geq 1$,
\[
\mathbb{E} \left\| \frac{1}{N} \sum_{t=0}^{N-1} g(Z_t) \right\|^2
\leq \frac{16 C_1^2 \sigma^2}{N}.
\]
\end{lemma}

Building on this, we show that the MLMC estimator achieves optimal bias-variance tradeoffs:

\begin{lemma}[MLMC Estimator Properties]
\label{lemma: mlmc statements}
Consider a unichain Markov chain $(Z_t)_{t\ge0}$ with a unique stationary distribution. Let $g^j$ be the time-averaged estimate of a function $\nabla F$ using a trajectory of length $2^j$. If $Q \sim \text{Geom}(1/2)$, the MLMC estimator defined as:
\begin{align*}
    g_{\text{MLMC}} = g^0 + \mathbb{I}\{2^Q \le T_{\max}\} \cdot 2^Q (g^Q - g^{Q-1})
\end{align*}
satisfies the following inequalities:
\begin{align}
    &\mathbb{E} \left[g_{\mathrm{MLMC}}\right] = \mathbb{E} \left[g^{\lfloor \log_2 T_{\max} \rfloor}\right] \\
    &\mathbb{E}\!\left[\|\nabla F(x) - g_{\mathrm{MLMC}}\|^2\right]\!\leq\!\mathcal{O}\!\left(\sigma^2 C_1^2 \log T_{\max}\!+ \delta^2\right) \\
    &\|\nabla F(x) - \mathbb{E}[g_{\mathrm{MLMC}}]\|^2 \leq \mathcal{O}\!\left(\sigma^2 C_1^2 T_{\max}^{-1} + \delta^2\right).
\end{align}
where \(C_1\) is \(C_\mathrm{tar}\) or \(C\) depending on whether the initial state is within the recurrent class \(\mathcal{S}_R^\theta\).
\end{lemma}

\paragraph{Technical Novelty}
The primary contribution of Lemma~\ref{lemma: mse unichain markov chain} is the extension of MLMC estimation beyond ergodic settings to the more general unichain case. Unlike ergodic chains where exponential mixing directly governs bias and variance through the mixing time $t_\mathrm{mix}$ \citep{beznosikov2023first, ganesh2024sharper, xu2025global}, unichain MDPs exhibit transient states and periodicity that prevent global uniform exponential mixing. Our analysis exploits the Poisson equation decomposition:
\begin{align}
    g(Z_t) = h(Z_t) - \mathbb{E}\left[h(Z_{t+1}) | \mathcal{F}_t\right], \quad (I - P)h = g
\end{align}
where \(\mathcal{F}_t\) denotes the filtration up to time \(t\). By telescoping this relation and leveraging uniform bounds on \(\|h\|_\infty\), we characterize bias and variance through the unichain constants \(C_\mathrm{tar}\) and \(C\) alone, without requiring exponential mixing.

\subsubsection{Handling Transient States via Explicit Burn-in}

The critic and NPG estimation procedures can be interpreted as stochastic linear systems driven by Markovian noise. In unichain settings, the kernel of the critic matrix $\mathbf{A}_g(\theta)$ depends on which transient states are induced by $\pi_\theta$, and thus varies with $\theta$. This violates the kernel inclusion condition necessary for tight bias control, which holds naturally in ergodic MDPs.

We resolve this issue by introducing an explicit burn-in phase at the beginning of each epoch. Before sampling for critic or NPG updates in epoch $k$, we execute $B$ steps under $\pi_{\theta_k}$, allowing the chain to reach the recurrent class $\mathcal{S}^{\theta_k}_R$. Define $T^{\theta_k}$ as the first hitting time to this recurrent class. The following lemma quantifies the impact of burn-in on regret:

\begin{lemma}[Regret decomposition under burn-in]
\label{lemma: regret-burnin}
Let $\mathcal{E}_B := \bigcap_{k=0}^{K-1} \{ T^{\theta_k} \le B \}$ denote the event that, in every epoch, the burn-in of length $B$ under policy $\pi_{\theta_k}$ reaches the recurrent class. Then the expected regret admits the decomposition
\begin{align*}
    \mathbb{E}[\mathrm{Reg}_T] \le \mathbb{E}[\mathrm{Reg}_T \mid \mathcal{E}_B] + \mathcal{O}\!\left( \frac{T^2 2^{-B/(2C_{\mathrm{hit}})}}{\log T} \right)
\end{align*}
\end{lemma}

Choosing $B = c \log T$ with $c > 4C_{\mathrm{hit}}$ ensures the second term becomes negligible for large $T$. Conditional on $\mathcal{E}_B$, each epoch effectively begins within the recurrent class, and the Strong Markov Property enables analyzing critic and NPG updates entirely within this class where kernel inclusion holds.

\paragraph{Technical Novelty.}
To minimize the probability of the initial state being outside the recurrent class, \citet{ganesh2025regret} employed $\mathcal{O}(\sqrt{T})$ burn-in steps. However, directly applying this in the CMDP setting yields a total burn-in cost of $KB = \Theta(T^{3/2}/\log T)$, which precludes order-optimal regret. Our analysis demonstrates that logarithmic burn-in $\mathcal{O}(\log T)$ suffices when combined with polynomial decay bounds. This drastically reduces sample complexity while preserving kernel inclusion, allowing our MLMC-based critic and NPG updates to achieve order-optimal scaling in large-scale unichain CMDPs.

\subsection{Estimation Procedures}

We now detail the construction of $\omega_k$ and $\eta^k_c$ using a Multi-Level Monte Carlo actor-critic framework, building off \citet{xu2025global}. The algorithm operates over $K$ epochs (outer loops), with epoch $k$ beginning from parameters $(\theta_k, \lambda_k)$. For expositional clarity, we present critic estimation first, as the resulting estimates feed into the NPG computation.

\subsubsection{Critic Estimation}

The critic estimation phase pursues two objectives at epoch $k$: (i) estimating the average reward $J_g(\theta_k)$ for $g \in \{r, c\}$, and (ii) estimating the value function $V^{\pi_{\theta_k}}_g$. 

The average reward can be characterized as the solution to the optimization problem:
\begin{align*}
    \min_{\eta \in \mathbb{R}} R_g(\theta_k, \eta) := \frac{1}{2} \mathbb{E}\left[(g(s,a) - \eta)^2 \right]
\end{align*}
where the expectation is over the state-action occupancy measure. For the value function, we adopt a linear approximation architecture: we assume $V^{\pi_{\theta_k}}_g(\cdot)$ is well-approximated by $\hat{V}_g(\zeta^{\theta_k}_g, \cdot) := \langle \phi_g(\cdot), \zeta^{\theta_k}_g \rangle$, where $\phi_g : \mathcal{S} \to \mathbb{R}^m$ is a feature map satisfying $\|\phi_g(s)\| \leq 1$ for all $s \in \mathcal{S}$. The critic parameter $\zeta^{\theta_k}_g \in \mathbb{R}^m$ is obtained via:
\begin{align*}
\min_{\zeta \in \mathbb{R}^m} E_g(\theta_k, \zeta) := \mathbb{E} \left[\frac{1}{2} \left(V^{\pi_{\theta_k}}_g(s) - \hat{V}_g(\zeta, s)\right)^2 \right]
\end{align*}
where the expectation is taken over the stationary distribution \(d^{\pi_{\theta_k}}\). 

The gradients of these objectives can be expressed as:
\begin{align}
    \nabla_\eta R_g(\theta_k, \eta) &= \mathbb{E}_{(s, a) \sim \nu^{\pi_{\theta_k}}} \left[\eta - g(s,a)\right], \forall \eta \in \mathbb{R} \label{eqn: grad_avg}\\
    \nabla_\zeta E_g(\theta_k, \zeta) &= \mathbb{E}_{(s, a) \sim \nu^{\pi_{\theta_k}}} \left(\zeta^\top \phi_g(s) - Q^{\pi_{\theta_k}}_g(s,a)\right) \nonumber \\&\qquad \phi_g(s), \forall \zeta \in \mathbb{R}^m \label{eqn: grad_critic}
\end{align}
Substituting the Bellman equation~\eqref{eqn: Bellman Equation} in equation~\eqref{eqn: grad_critic}\ yields:
\begin{align}
    &\nabla_\zeta E_g(\theta_k, \zeta) = \mathbb{E}_{(s, a) \sim \nu^{\pi_{\theta_k}}} \Big[ \Big(\zeta^\top \phi_g(s) - g(s,a) \nonumber \\
    &\quad + J_g(\theta_k) - \mathbb{E}_{s' \sim P(\cdot|s,a)}[V^{\pi_{\theta_k}}_g(s')]\Big) \phi_g(s) \Big] \label{eqn: grad_critic_bellman}
\end{align}

Since exact computation of these gradients is infeasible without knowledge of $P$ and $\nu^{\pi_{\theta_k}}_g$, we perform $H$ inner loop iterations starting from $\eta^k_{g,0} = 0$ and $\zeta^k_{g,0} = 0$:
\begin{align}
    \eta^k_{g,h+1} &= \eta^k_{g,h} - c_\gamma \gamma_\xi \hat{\nabla}_\eta R_g(\theta_k, \eta^k_{g,h}) \nonumber \\ 
    \zeta^k_{g,h+1} &= \zeta^k_{g,h} - \gamma_\xi \hat{\nabla}_\zeta E_g(\theta_k, \xi^k_{g,h}) \label{eqn: critic_update}
\end{align}
where $c_\gamma$ is a constant, $\gamma_\xi$ is the learning rate, $h \in \{0, \ldots, H-1\}$, and $\hat{\nabla}_\eta R_g(\theta_k, \eta^k_{g,h})$, $\hat{\nabla}_\zeta E_g(\theta_k, \xi^k_{g,h})$ denote estimates of the corresponding gradients. Here we define $\xi^k_{g,h} \triangleq [(\eta^k_{g,h})^\top, (\zeta^k_{g,h})^\top]^\top$. Note that the gradient $\nabla_\zeta E_g(\theta_k, \zeta^k_{g,h})$ in~\eqref{eqn: grad_critic_bellman} depends on both $\zeta^k_{g,h}$ and $J_g(\theta_k)$, making its estimate a function of both $\eta^k_{g,h}$ and $\zeta^k_{g,h}$. After $H$ iterations, we obtain $\eta^k_{g,H}$ and $\zeta^k_{g,H}$ as estimates of $J_g(\theta_k)$ and $\zeta^{\theta_k}_g$, respectively.

We can write the updates~\eqref{eqn: critic_update} compactly as:
\begin{align}
    \xi^k_{g,h+1} = \xi^k_{g,h} - \gamma_\xi v_g(\theta_k, \xi^k_{g,h}), \quad v_g(\theta_k, \xi^k_{g,h}) \triangleq \nonumber \\ 
    [c_\gamma \hat{\nabla}_\eta R_g(\theta_k, \eta^k_{g,h}), (\hat{\nabla}_\zeta E_g(\theta_k, \xi^k_{g,h}))^\top]^\top \label{eqn: critic_compact}
\end{align}

For a given pair $(k,h)$, let $\mathcal{T}_{kh} = \{(s^t_{kh}, a^t_{kh}, s^{t+1}_{kh})\}_{t=0}^{l_{kh}-1}$ denote a trajectory of length $l_{kh}$ induced by $\pi_{\theta_k}$. From a single transition $z^j_{kh} = (s^j_{kh}, a^j_{kh}, s^{j+1}_{kh})$, we construct an estimate of $v_g(\theta_k, \xi^k_{g,h})$ following~\eqref{eqn: grad_avg} and~\eqref{eqn: grad_critic_bellman}:
\begin{align}
    &v_g(\theta_k, \xi^k_{g,h}; z^j_{kh}) = \mathbf{A}_{g}(\theta_k; z_{kh}^{j})\xi^k_{g,h} - \mathbf{b}_{g}(\theta_k; z_{kh}^{j}), \label{eqn: single_sample_critic}\\
    &\mathbf{A}_{g} = \begin{bmatrix} c_\gamma & 0 \\ \phi_g(s^j_{kh}) & \phi_g(s^j_{kh})(\phi_g(s^j_{kh}) - \phi_g(s^{j+1}_{kh}))^\top \end{bmatrix}, \nonumber \\
    &\mathbf{b}_{g} = \begin{bmatrix} c_\gamma g(s^j_{kh}, a^j_{kh}) \\ g(s^j_{kh}, a^j_{kh})\phi_g(s^j_{kh}) \end{bmatrix}. \nonumber
\end{align}

To reduce variance while maintaining unbiasedness, we apply the MLMC framework. At inner iteration $h$, we sample $Q^k_h \sim \text{Geom}(1/2)$ and generate a trajectory of length $l_{kh} = (2^{Q^k_h} - 1) \cdot \mathbb{I}(2^{Q^k_h} \leq T_{\max}) + 1$, where $T_{\max}$ is a predetermined maximum. For each level $j \in \{0, Q^k_h - 1, Q^k_h\}$, we compute:
\begin{align}
    v^j_{g,kh} = \frac{1}{2^j} \sum_{t=0}^{2^j - 1} v_g(\theta_k, \xi^k_{g,h}; z^t_{kh}) \label{eqn: time_avg_critic}
\end{align}
The MLMC estimate combines these levels via:
\begin{align}
    &\hat{v}_g(\theta_k, \xi^k_{g,h}) \nonumber \\
    &= v^0_{g,kh} + \begin{cases} 2^{Q^k_h}(v^{Q^k_h}_{g,kh} - v^{Q^k_h - 1}_{g,kh}), & 2^{Q^k_h} \leq T_{\max} \\ 0, & \text{otherwise} \end{cases}, \label{eqn: mlmc_critic}
\end{align}
where the expected number of samples is $\mathbb{E}[l_{kh}] = O(\log T_{\max})$. The MLMC estimator achieves bias comparable to naively averaging $T_{\max}$ samples while using only $O(\log T_{\max})$ samples in expectation. Moreover, since the geometric distribution does not require explicit knowledge of hitting times $C_{\mathrm{hit}}$ or target times $C_{\mathrm{tar}}$, (only an upper bound), our approach eliminates the need for these constants to be known a priori.

\subsubsection{Natural Policy Gradient (NPG) Estimation}

The critic estimation phase outputs $\xi^k_{g,H} = [(\eta^k_{g,H})^\top, (\zeta^k_{g,H})^\top]^\top$, where $\eta^k_{g,H}$ and $\zeta^k_{g,H}$ estimate $J_g(\theta_k)$ and the critic parameter $\zeta^{\theta_k}_g$, respectively. For notational convenience, we write $\xi^k_{g,H}$ as $\xi^k_g = [(\eta^k_g)^\top, (\zeta^k_g)^\top]^\top$. We estimate the NPG $\omega^*_{g,\theta_k}$ (defined in~\eqref{eqn: optimal npg direction}) through an $H$-step inner loop initialized at $\omega^k_{g,0} = 0$:
\begin{align}
    \omega^k_{g,h+1} = \omega^k_{g,h} - \gamma_\omega \hat{\nabla}_\omega f_g(\theta_k, \omega^k_{g,h}, \xi^k_g),
    \label{eqn: npg_update}
\end{align}
for \(h \in \{0, \ldots, H-1\}\), where $\hat{\nabla}_\omega f_g(\theta_k, \omega^k_{g,h}, \xi^k_g)$ is an MLMC-based estimate of $\nabla_\omega f_g(\theta_k, \omega^k_{g,h})$, and $f_g$ is defined in~\eqref{eqn: npg convex optimization}. To construct this estimate, we generate a $\pi_{\theta_k}$-induced trajectory $\mathcal{T}_{kh} = \{(s^t_{kh}, a^t_{kh}, s^{t+1}_{kh})\}_{t=0}^{l_{kh}-1}$ of length $l_{kh} = (2^{P^k_h} - 1) \cdot \mathbb{I}(2^{P^k_h} \leq T_{\max}) + 1$, where $P^k_h \sim \text{Geom}(1/2)$. For a single transition $z^j_{kh} = (s^j_{kh}, a^j_{kh}, s^{j+1}_{kh})$, we define:
\begin{align}
    \hat{\nabla}_\omega f_g(\theta_k, \omega^k_{g,h}, \xi^k_g; z^j_{kh}) \nonumber \\
    = \hat{F}(\theta_k; z^j_{kh})\omega^k_{g,h} - \hat{\nabla}_\theta J_g(\theta_k, \xi^k_g; z^j_{kh}) \label{eq:single_npg}
\end{align}
where
\begin{align}
    &\hat{F}(\theta_k; z^j_{kh}) \nonumber \\ 
    &= \nabla_\theta \log \pi_{\theta_k}(a^j_{kh}|s^j_{kh}) \otimes \nabla_\theta \log \pi_{\theta_k}(a^j_{kh}|s^j_{kh}), \label{eq:fisher_est} \\
    &\hat{\nabla}_\theta J_g(\theta_k, \xi^k_g; z^j_{kh})  \nonumber \\ 
    &= \hat{A}^{\pi_{\theta_k}}_g(\xi^k_g; z^j_{kh}) \nabla_\theta \log \pi_{\theta_k}(a^j_{kh}|s^j_{kh}), \label{eq:pg_est} \\
    &\hat{A}^{\pi_{\theta_k}}_g(\xi^k_g; z^j_{kh}) \nonumber \\ 
    &= g(s^j_{kh}, a^j_{kh}) - \eta^k_g + (\zeta^k_g)^\top (\phi_g(s^{j+1}_{kh}) - \phi_g(s^j_{kh})). \label{eq:adv_est}
\end{align}
where the advantage estimate in~\eqref{eq:adv_est} is a temporal difference (TD) error that relies on $\xi^k_g$ from the critic phase. We construct the MLMC estimate analogously to the critic case. For each level $j \in \{0, P^k_h - 1, P^k_h\}$, we compute:
\begin{align}
    u^j_{g,kh} = \frac{1}{2^j} \sum_{t=0}^{2^j - 1} \hat{\nabla}_\omega f_g(\theta_k, \omega^k_{g,h}, \xi^k_g; z^t_{kh}) \label{eqn: time_avg_npg}
\end{align}
The MLMC estimate is then:
\begin{align}
    &\hat{\nabla}_\omega f_g(\theta_k, \omega^k_{g,h}, \xi^k_g) \nonumber \\ 
    &= u^0_{g,kh} + \begin{cases} 2^{P^k_h}(u^{P^k_h}_{g,kh} - u^{P^k_h - 1}_{g,kh}), & 2^{P^k_h} \leq T_{\max} \\ 0, & \text{otherwise} \end{cases} \label{eqn: mlmc_npg}
\end{align}
Applying this estimate in~\eqref{eqn: npg_update} over $H$ iterations yields the NPG estimate $\omega^k_{g,H}$. For notational simplicity, we denote $\omega^k_{g,H}$ as $\omega^k_g$.

\subsection{Parameter Updates}

Having obtained $\omega^k_g$ for $g \in \{r,c\}$ and $\eta^k_c$ from the estimation procedures, we form the combined NPG estimate $\omega_k = \omega^k_r + \lambda_k \omega^k_c$, where $\eta^k_c$ estimates $J_c(\theta_k)$. At epoch $k$, the policy and dual parameters are updated according to~\eqref{eq:approx_update}:
\begin{align}
\theta_{k+1} = \theta_k + \alpha \omega_k, \quad \lambda_{k+1} = \mathcal{P}_{\left[0, \frac{2}{\delta}\right]}\left[\lambda_k - \beta \eta^k_c\right] \label{eq:final_update}
\end{align}

The complete algorithm, incorporating burn-in, MLMC estimation for critic and NPG, and the primal-dual updates, is presented in Algorithm~\ref{alg: pdnac_bi}.

%% file: Sections/assumptions.tex
\section{Assumptions}
\label{sec: assumptions}
In this section, we formally state the assumptions required for our analysis. Our assumptions are natural extensions of those used in the unconstrained setting \citep{ganesh2025regret} to the constrained MDP framework.

\subsection{Critic Approximation}

We employ function approximation for both the reward and cost critics. For the critics, we use $\hat{V}_g(\zeta^g, s) = (\zeta^g)^\top \phi^g(s)$, where $\zeta^g \in \mathbb{R}^{m_g}$ are the critic parameters and $\phi^g : \mathcal{S} \to \mathbb{R}^{m_g}$ is a feature mapping for \(g = \{r, c\}\). We also assume that the feature vectors are bounded and linearly independent from each other \citep{ganesh2025regret}. The quality of these approximations is characterized by the following error measures:

\begin{definition}[Critic Approximation Errors]
\label{def:critic_approx}
We define the critic approximation error for \(g \in \{r, c\}\) as
\begin{equation}
    \epsilon_{\mathrm{app}}^g := \sup_{\theta \in \Theta} \inf_{\zeta^g \in \mathbb{R}^{m_g}} \frac{1}{2} \mathbb{E}_{s \sim d^{\pi_\theta}} \left[ \left( V_g^{\pi_\theta}(s) - \hat{V}_g(\zeta^g, s) \right)^2 \right],
\end{equation}
\end{definition}

These quantities measure the inherent approximation error resulting from the choice of feature space. A well-designed feature map yields small (or even vanishing) approximation error. This definition is standard in the actor-critic literature \citep{suttle2023beyond,chen2023finite,patel2024towards,ganesh2024sharper,xu2025global,wang2024non}.

For each policy $\pi_\theta$, we define the critic feature matrices
\begin{equation}
    M_\theta^g := \mathbb{E}_\theta \left[ \phi^g(s) \left( \phi^g(s) - \phi^g(s') \right)^\top \right]
\end{equation}
where the expectation $\mathbb{E}_\theta$ is taken over $s \sim d^{\pi_\theta}$ and $s' \sim P^{\pi_\theta}(s, \cdot)$. These matrices appear in the temporal-difference updates for the critics and play a crucial role in convergence analysis.

\begin{assumption}[Critic Matrix Lower Bounds]
\label{assumption: critic_matrix}
There exist constants $\lambda_g > 0$ such that for all $\theta \in \Theta$:
\begin{equation}
    x^\top M_\theta^g x \geq \lambda_g \|x\|^2, \quad \forall x \in \ker(M_\theta^r)^\perp,
\end{equation}
\end{assumption}

\noindent \textbf{Comments on Assumption~\ref{assumption: critic_matrix}.} This assumption is substantially weaker than the commonly imposed requirement of strict positive-definiteness, which appears in nearly all actor-critic works for ergodic MDPs \citep{suttle2023beyond, patel2024towards, wang2024non, panda2025two}. In the unichain setting, strict positive-definiteness of the critic matrices cannot be guaranteed in general, as the kernel structure depends on the transient states under the policy (see Section 4.2 in \citet{ganesh2025regret}). Our analysis accommodates this by working with projected iterates on the orthogonal complement of the kernel. This flexibility is essential for handling the more general unichain assumption, which permits both transient states and periodic behavior.

\subsection{Policy Parameterization}

We consider a parameterized policy class $\Pi = \{\pi_\theta : \theta \in \Theta\}$ where $\Theta \subseteq \mathbb{R}^d$. The following assumptions on the policy parameterization are standard in the policy gradient literature.

\begin{assumption}[Bounded and Lipschitz Score Function]
\label{assumption: score bounds}
For all $\theta, \theta_1, \theta_2 \in \Theta$ and $(s, a) \in \mathcal{S} \times \mathcal{A}$, the following hold:
\begin{align}
    \|\nabla_\theta \log \pi_\theta(a|s)\| \leq G_1 \\
    \|\nabla_\theta \log \pi_\theta(a|s) - \nabla_\theta \log \pi_{\theta'}(a|s)\| \leq G_2 \|\theta - \theta'\|
\end{align}
\end{assumption}

\begin{assumption}[Fisher Non-Degeneracy]
\label{assumption: fisher}
There exists a constant $\mu > 0$ such that $F(\theta) \succeq \mu I_d$ for all $\theta \in \Theta$, where $I_d$ denotes the $d \times d$ identity matrix.
\end{assumption}

\noindent \textbf{Comments on Assumptions~\ref{assumption: score bounds}--\ref{assumption: fisher}.} These are standard assumptions in the policy gradient literature \citep{liu2020improved, fatkhullin2023stochastic,masiha2022stochastic,mondal2024improved,wang2024non,suttle2023beyond,ganesh2025order,ganesh2025regret}. Assumption~\ref{assumption: score bounds} stipulates that the score function is both bounded and Lipschitz-continuous, which is frequently used in error decomposition and policy gradient estimation. Assumption~\ref{assumption: fisher} requires that the eigenvalues of the Fisher information matrix are uniformly bounded away from zero, ensuring that the natural policy gradient direction is well-defined. These assumptions have been verified for various policy classes, including Gaussian and Cauchy distributions with parameterized means and clipping, as well as certain neural parameterizations \citep{liu2020improved,fatkhullin2023stochastic}.

The expressivity of the policy class is characterized by the following error measure:

\begin{definition}[Policy Approximation Error]
\label{def:policy_bias}
Define $\epsilon_{\mathrm{bias}}$ as the least upper bound on the transferred compatible function approximation error:
\begin{align}
    \epsilon_{\mathrm{bias}} := \sup_{\theta \in \Theta} L_{\nu^{\pi_\theta^*}}(\omega_\theta^*; \theta),
\end{align}
\end{definition}
\noindent \textbf{Comments on $\epsilon_{\mathrm{bias}}$.} The term $\epsilon_{\mathrm{bias}}$ is standard in the literature on parameterized policy gradient methods \citep{agarwal2020optimality,fatkhullin2023stochastic,masiha2022stochastic,mondal2024improved,suttle2023beyond,ganesh2025order,ganesh2025regret,xu2019sample} and reflects the expressivity of the chosen policy class. When the policy class is sufficiently expressive to represent any stochastic policy, such as with softmax parameterization over a tabular space, we have $\epsilon_{\mathrm{bias}} = 0$ \citep{agarwal2021theory}. For more restrictive parameterizations that do not cover all stochastic policies, $\epsilon_{\mathrm{bias}} > 0$. Nevertheless, this bias is often negligible when using rich neural network parameterizations \citep{wang2019neural}. In the constrained setting, we take the maximum over both errors to ensure uniform convergence guarantees.

\noindent \textbf{Notation.} For convenience, we define $\epsilon_{\mathrm{app}} := \max\{\epsilon_{\mathrm{app}}^r, \epsilon_{\mathrm{app}}^c\}$ and $\lambda := \min\{\lambda_r, \lambda_c\}$ to simplify the presentation of our main results.

%% file: Sections/regret_analysis.tex
\section{Theoretical Analysis}
\label{sec:analysis}

In this section, we provide the convergence analysis of Algorithm~\ref{alg: pdnac_bi}. Our analysis proceeds in three steps: first, we establish the convergence of the critic parameters using the Multi-Level Monte Carlo (MLMC) estimator; second, we bound the estimation error of the Natural Policy Gradient (NPG) direction; finally, we combine these results to establish global regret bounds for the reward objective and cumulative constraint violation.

\subsection{Convergence of MLMC Critic Estimation}

A central challenge in unichain CMDPs is the presence of transient states, which can introduce significant bias if not handled correctly. Our algorithm utilizes a burn-in period $B$ to ensure the process reaches the recurrent class $\mathcal{S}_R^{\theta_k}$ (Assumption~\ref{assumption: unichain}), followed by an MLMC estimator to compute the critic updates. We first show that under the Slater condition (Assumption~\ref{assumption: slater}) and the critic matrix lower bound (Assumption~\ref{assumption: critic_matrix}), the critic parameters $\xi_{g, H}^k$ converge to the optimal projected parameters $\xi_g^*(\theta_k)$.

\begin{theorem}[Unichain Critic Convergence with MLMC]
    \label{theorem: unichain_critic_final}
    Consider the critic estimation subroutine in Algorithm 1 for a unichain CMDP. Suppose the assumptions of Section~\ref{sec: assumptions} hold. Then conditioned on the event $\mathcal{E}_B$, the iterates satisfy:
    \begin{align*}
        \mathbb{E}_k [\|\Pi(\xi_{g, H}^k - \xi^*_g(\theta_k))\|^2 \mid \mathcal{E}_B] \nonumber &\leq \tilde{\mathcal{O}} \left(\frac{c_\gamma^2}{\lambda^2 T^2} + \frac{c_\gamma^4 C_{\mathrm{tar}}^2 }{\lambda^4 T_\mathrm{max}}\right) \\
        \|\Pi(\mathbb{E}_k[\xi_{g, H}^k \mid \mathcal{E}_B] - \xi^*_g(\theta_k))\|^2 &\leq \tilde{\mathcal{O}} \left(\frac{c_\gamma^2}{\lambda^2 T^2} + \frac{c_\gamma^4 C_{\mathrm{tar}}^2 }{\lambda^4 T_\mathrm{max}}\right)
    \end{align*}
\end{theorem}
The proof (detailed in Appendix~\ref{sec: critic analysis}) relies on characterizing the MLMC estimator's bias-variance tradeoff. By selecting a maximum trajectory length $T_{\max}$, the MLMC estimator achieves a bias comparable to a rollout of length $T_{\max}$ while maintaining low variance. We then use Theorem~\ref{theorem: unichain stochastic linear recursion} to bound the final critic error using stochastic linear recursions.

\subsection{Convergence of Natural Policy Gradient Estimates}

Building on the critic convergence, we analyze the NPG update. The NPG direction $\omega_k$ is estimated by solving a ridge regression problem involving the Fisher Information Matrix (FIM) and the gradient of the value function.

\begin{theorem}[NPG Convergence for Unichain CMDPs]
\label{theorem: npg_convergence}
Consider Algorithm~\ref{alg: pdnac_bi} under the setting of Theorem~\ref{theorem: unichain_critic_final} with step size $\gamma_\omega = \frac{8 \log T}{\mu H}$. Then, conditioned on the event $\mathcal{E}_B$:
\begin{align*}
    \mathbb{E}_k &[\|\omega_g^k - \omega_{g,k}^*\|^2] \\ 
    &\leq \mathcal{O}\left(\frac{G_1^2 C^2 c_\gamma^2}{\mu^2 T^2 \lambda^2} + \frac{G_1^6 C_{\mathrm{tar}}^2 c_\gamma^4 C^2}{\mu^4 \lambda^4} + \frac{G_1^2 \epsilon_{\mathrm{app}}}{\mu^2}\right)
\end{align*}
\begin{align*}
    |\mathbb{E}_k &[\omega_g^k] - \omega_{g,k}^*\|^2 \\ 
    &\leq \tilde{\mathcal{O}}\left(\frac{G_1^2 c_\gamma^2 C_{\mathrm{tar}}^4}{\mu^2 \lambda^2 T^2} + \frac{G_1^{10} C_{\mathrm{tar}}^4 c_\gamma^2 C^2}{\mu^6 T_{\mathrm{max}} \lambda^4} + \frac{G_1^2 \epsilon_{\mathrm{app}}}{\mu^2}\right)
\end{align*}
\end{theorem}

This result (proven in Appendix~\ref{sec: npg analysis}) demonstrates that the error in the NPG direction is controlled by the critic approximation error $\epsilon_{\mathrm{app}}$ and terms that decay with the number of inner loop iterations $H = \Theta(\log T)$ and the effective trajectory length $T_{\max}$. Similar to the critic, we use stochastic linear recursions to bound NPG direction error.

\subsection{Regret and Constraint Violation Bounds}
Finally, we present the main result of this work: the finite-time regret analysis for the unichain CMDP. By leveraging the primal-dual framework, we bound the regret of the reward function and the cumulative violation of the constraint.

\begin{theorem}[Main Regret Bounds]
\label{theorem: main_regret}
Consider Algorithm 1 with loop sizes $H = \Theta(\log T)$ and $K = \Theta(T/\log T)$, and step sizes $\alpha, \beta = \Theta(1/\sqrt{T})$. Under the assumptions in Section~\ref{sec: assumptions}, the expected regret and cumulative constraint violation satisfy:
\begin{align*}
    &\mathbb{E}[\mathrm{Reg}_T] = \tilde{\mathcal{O}}\left( T(\sqrt{\epsilon_{\mathrm{bias}}} + \sqrt{\epsilon_{\mathrm{app}}}) + C_{\mathrm{tar}}^2C \sqrt{T} \right) \\
    &\mathbb{E}\left[ \sum_{t=0}^{T-1} J_c(\theta_t) \right] = \tilde{\mathcal{O}}\left( T(\sqrt{\epsilon_{\mathrm{bias}}} + \sqrt{\epsilon_{\mathrm{app}}}) + C_{\mathrm{tar}}^2C \sqrt{T} \right)
\end{align*}
\end{theorem}

\paragraph{Proof Sketch of Theorem~\ref{theorem: main_regret}}
We provide the full proof in Appendix~\ref{sec: unichain cmdp regret analysis proofs}. The proof begins by decomposing the total regret into a Lagrangian primal-dual gap and a drift term associated with the value function. We analyze this decomposition conditioned on the high-probability event $\mathcal{E}_B$, which guarantees that the burn-in period $B = \Theta(\log T)$ successfully places the system within the recurrent class. Inside this event, the unichain MDP behaves effectively like an ergodic chain on the recurrent states. We utilize the convergence bounds derived for the MLMC-based critic (Theorem~\ref{theorem: unichain_critic_final}) and the NPG estimator (Theorem~\ref{theorem: npg_convergence}) to bound the optimization error of the Lagrangian function by $\tilde{\mathcal{O}}(1/\sqrt{T})$ plus the inherent approximation errors $\epsilon_{\text{bias}}$ and $\epsilon_{\text{app}}$. Crucially, the stability of the dual updates ensures that the constraint violation is bounded by the regret of the Lagrangian plus a term decaying as $O(1/\sqrt{T})$. Combining the rapid decay of the burn-in failure probability with the finite-time convergence of the primal-dual iterates yields the final regret and violation bounds, which depend on the unichain complexity constants $C_{\text{hit}}$ and $C_{\text{tar}}$ rather than the mixing time.

%% file: Sections/conclusion.tex
\section{Conclusion.}
In this paper, we provided the first order-optimal regret analysis for infinite-horizon average-reward Constrained MDPs with general policy and critic parameterizations under the unichain assumption. By developing a Multi-Level Monte Carlo estimator specifically designed for unichain dynamics and combining it with a logarithmic burn-in strategy, we avoided reliance on mixing-time oracles and restrictive ergodicity assumptions. Our Primal-Dual Natural Actor-Critic algorithm achieves $\tilde{\mathcal{O}}(\sqrt{T})$ regret and constraint violation, establishing the first theoretical guarantees for general parameterized CMDPs in the unichain setting.

%% file: Sections/Appendix/related_work.tex
\section{Related Works}
\label{sec: related work}
\subsection{Average-Reward Constrained MDPs.}
Average-reward CMDPs have been studied extensively, starting from tabular and model-based formulations in the ergodic setting, where exponential mixing ensures aperiodicity and irreducibility. Early work leveraged linear programming or posterior sampling techniques to obtain $\tilde{\mathcal{O}}(\sqrt{T})$ regret guarantees \citep{chen2022learning, singh2020learning,agarwal2022regret,agarwal2022concave}. To handle larger state spaces, subsequent efforts shifted toward model-free, policy-based algorithms. In the tabular setting, \citet{ding2020natural} and \citet{wei2022provably} established convergence guarantees for natural policy gradient and primal-dual methods, while \citet{ghosh2023achieving} achieved sublinear regret with zero constraint violation for linear CMDPs.

The setting with general policy parametrization is substantially more challenging. Without the constraints, the problem has been recently studied in \cite{bai2024regret,patel2024towards,ganesh2024sharper}. 
\citet{bai2024learning} provided the first global convergence guarantees for primal-dual policy gradient methods, though with a suboptimal rate of $\tilde{\mathcal{O}}(T^{-1/5})$. More recently, \citet{xu2025global} achieved the order-optimal $\tilde{\mathcal{O}}(1/\sqrt{T})$ rate using a Natural Actor-Critic (NAC) framework combined with variance-reduced gradient estimation. However, all of these results rely critically on an ergodicity assumption, which ensures uniform mixing times and bounded hitting times across policies. This assumption excludes unichain MDPs with transient states, where the induced stationary distribution and mixing behavior depend on the policy. Our work follows the primal-dual NPG paradigm of \citet{xu2025global} but removes the ergodicity requirement, extending order-optimal guarantees to the unichain setting.

\subsection{Unichain and Non-Ergodic MDPs.}
Relaxing ergodicity is a longstanding challenge in average-reward reinforcement learning. The unichain assumption, which permits transient states and periodic behavior, is the weakest condition under which the average-reward Bellman equations are well-defined \citep{puterman2014markov}. In tabular settings, value-based methods such as UCRL2 \citep{jaksch2010near} and its refinements \citep{zhang2023sharper} naturally handle unichain dynamics. More recently, \citet{agrawal2024optimistic} and \citet{li2025stochastic} achieved near-optimal guarantees for unichain MDPs using optimistic or generative-model-based approaches, though these methods do not extend to general policy parameterizations.

For unichain methods with general function approximation, progress has been limited. The most closely related work is \citet{ganesh2025regret}, which established $\tilde{\mathcal{O}}(\sqrt{T})$ regret bounds for unconstrained unichain MDPs. Their analysis relies on large batch sizes $B = \Theta(\sqrt{T})$ to mitigate periodicity and control bias. While effective in unconstrained settings, such batching is incompatible with constrained optimization, where delayed dual updates can lead to significant constraint violation. Our work addresses this gap by developing variance-reduced gradient estimators tailored to unichain dynamics, enabling batch sizes of only $\mathcal{O}(\log T)$ and supporting stable primal-dual updates in CMDPs.

%% file: Sections/Appendix/algorithm.tex
\section{Algorithm}
\begin{algorithm}[H]
\caption{Unichain Primal-Dual Natural Actor-Critic with Burn-In (PDNAC-BI)}
\label{alg: pdnac_bi}
\begin{algorithmic}[1]
\STATE \textbf{Input:} Initial parameters $\theta_0$, $\lambda_0 = 0$, learning rates $\alpha, \beta, \gamma_\xi, \gamma_\omega$, critic parameter $c_\gamma$, initial state $s_0 \sim \rho$, loop sizes $K, H$, burn-in length $B$, maximum trajectory length $T_{\max}$
\FOR{$k = 0, 1, \ldots, K-1$}
    \STATE $\xi^k_{g,0} \leftarrow 0$, $\omega^k_{g,0} \leftarrow 0$
    \FOR{$b = 0, 1, \ldots, B-1$}
        \STATE Sample action $a^b_{k} \sim \pi_{\theta_k}(\cdot|s^b_{k})$
        \STATE Observe next state $s^{b+1}_{k} \sim P(\cdot|s^t_{kh}, a^t_{k})$
    \ENDFOR 
    \STATE $s_0 \leftarrow s^{B}_{k}$
    \FOR{$h = 0, 1, \ldots, H-1$}
        \STATE $s^0_{kh} \leftarrow s_0$
        \STATE Draw $Q^k_h \sim \text{Geom}(1/2)$, set $l_{kh} \leftarrow (2^{Q^k_h} - 1)\cdot \mathbb{I}(2^{Q^k_h} \leq T_{\max}) + 1$
        \FOR{$t = 0, 1, \ldots, l_{kh} - 1$}
            \STATE Sample action $a^t_{kh} \sim \pi_{\theta_k}(\cdot|s^t_{kh})$
            \STATE Observe $s^{t+1}_{kh} \sim P(\cdot|s^t_{kh}, a^t_{kh})$ and $g(s^t_{kh}, a^t_{kh})$
        \ENDFOR
        \STATE Compute MLMC estimate $\hat{v}_g(\theta_k, \xi^k_{g,h})$ using~\eqref{eqn: mlmc_critic}
        \STATE Update $\xi^k_{g,h+1} \leftarrow \xi^k_{g,h} - \gamma_\xi \hat{v_g}(\theta_k, \xi^k_{g,h})$
        \STATE $s_0 \leftarrow s^{l_{kh}}_{kh}$
    \ENDFOR
    \STATE $\xi^k_g \leftarrow \xi^k_{g,H}$
    \FOR{$h = 0, 1, \ldots, H-1$}
        \STATE $s^0_{kh} \leftarrow s_0$
        \STATE Draw $P^k_h \sim \text{Geom}(1/2)$, set $l_{kh} \leftarrow (2^{P^k_h} - 1)\cdot \mathbb{I}(2^{P^k_h} \leq T_{\max}) + 1$
        \FOR{$t = 0, 1, \ldots, l_{kh} - 1$}
            \STATE Sample action $a^t_{kh} \sim \pi_{\theta_k}(\cdot|s^t_{kh})$
            \STATE Observe $s^{t+1}_{kh} \sim P(\cdot|s^t_{kh}, a^t_{kh})$ and $g(s^t_{kh}, a^t_{kh})$
        \ENDFOR
        \STATE Compute $\hat{\nabla}_\omega f_g(\theta_k, \omega^k_{g,h}, \xi^k_g)$ using~\eqref{eqn: mlmc_npg}
        \STATE Update $\omega^k_{g,h+1} \leftarrow \omega^k_{g,h} - \gamma_\omega \hat{\nabla}_\omega f_g(\theta_k, \omega^k_{g,h}, \xi^k_g)$
        \STATE $s_0 \leftarrow s^{l_{kh}}_{kh}$
    \ENDFOR
    \STATE $\omega^k_g \leftarrow \omega^k_{g,H}$

    \STATE $\omega_k \leftarrow \omega^k_r + \lambda_k \omega^k_c$
    \STATE $\theta_{k+1} \leftarrow \theta_k + \alpha \omega_k$
    \STATE $\lambda_{k+1} \leftarrow \mathcal{P}_{[0, 2/\delta]}[\lambda_k - \beta \eta^k_c]$
\ENDFOR
\STATE \textbf{Return:} $\{\theta_k\}_{k=0}^K$
\end{algorithmic}
\end{algorithm}

%% file: Sections/Appendix/stochastic_linear_recursion.tex
\section{Unichain Stochastic Linear Recursions}
\label{sec: unichain stochastic linear recursions}
Consider a general stochastic linear recursion of the form:
\begin{equation}
    x_{h+1} = x_h - \bar{\beta}(P_h x_h - q_h)
\end{equation}
where $P_h$ and $q_h$ are noisy estimates of $P \in \mathbb{R}^{n \times n}$ and $q \in \mathbb{R}^n$, respectively, and $h \in \{0, \dots, H-1\}$. Let $\mathbb{E}_h$ denote the conditional expectation given the history up to step $h$.

\subsection*{Conditions (C1--C9)}
Assume that the following conditions hold for all $h$:
\begin{itemize}
    \item[\textbf{(C1)}] Variance of $P_h$: $\mathbb{E}_h [\|P_h - P\|^2] \leq \sigma_P^2$
    \item[\textbf{(C2)}] Bias of $P_h$: $\|\mathbb{E}_h [P_h] - P\|^2 \leq \delta_P^2$
    \item[\textbf{(C3)}] Variance of $q_h$: $\mathbb{E}_h [\|q_h - q\|^2] \leq \sigma_q^2$
    \item[\textbf{(C4)}] Bias of $q_h$: $\|\mathbb{E}_h [q_h] - q\|^2 \leq \delta_q^2$
    \item[\textbf{(C5)}] Total Bias of $q_h$: $\|\mathbb{E} [q_h] - q\|^2 \leq \bar{\delta}_q^2$
    \item[\textbf{(C6)}] Norm bound on $P$: $\|P\| \leq \Lambda_P$
    \item[\textbf{(C7)}] Norm bound on $q$: $\|q\| \leq \Lambda_q$
    \item[\textbf{(C8)}] Subspace Positive Definiteness: $x^\top P x \geq \lambda_P \|x\|^2$ for all $x \in \text{ker}(P)^\perp$
    \item[\textbf{(C9)}] Kernel Inclusion: $\text{ker}(P) \subseteq \text{ker}(P_h)$
\end{itemize}

\begin{theorem}[Theorem 4, \citep{ganesh2025regret}]
\label{theorem: unichain stochastic linear recursion}
Consider the recursion (1) and suppose (C1)--(C9) hold. Further, let $\delta_P \leq \lambda_P/8$. Then, the following bounds hold for all $H \geq 1$:

\begin{equation*}
    \mathbb{E} [\|\Pi(x_H - x^*)\|^2] \leq \mathcal{O} \left( \exp \left( -\frac{H \bar{\beta}\lambda_P}{4} \right) R_0^2 + \bar{\beta} \lambda_P^{-1}\left( \sigma_P^2 \lambda_P^{-2} \Lambda_q^2 + \sigma_q^2\right) + \lambda_P^{-2} \left(\delta_P^2 \lambda_P^{-2} \Lambda_q^2 + \delta_q^2\right) \right)
\end{equation*}

\begin{equation*}
    \|\Pi(\mathbb{E}[x_H] - x^*)\|^2 \leq \mathcal{O} \left( \exp \left( -\frac{H\bar{\beta}\lambda_P}{4} \right) R_0^2 + \left(\lambda_P^{-2} + \bar{\beta} \lambda_P^{-1}\right)\left( \delta_P^2 R_0^2 + \delta_P^2 \Lambda_P^2 \lambda_P^{-2} + \bar{\delta}_q^2\right)\right)
\end{equation*}
where $R_0 = \|\Pi(x_0 - x^*)\|^2$, $x^* = P^\dagger q$, and $\Pi$ denotes the projection operator onto the space $\text{ker}(P)^\perp$.
\end{theorem}

%% file: Sections/Appendix/supporting_lemmas.tex
\section{Unichain MLMC Estimator}
\label{sec: unichain mlmc estimator proofs}
\begin{lemma}[Lemma 2, \citep{ganesh2025regret}]
    \label{lemma: non recurrent time average}
    Let Assumption~\ref{assumption: unichain} hold and consider any $\theta \in \Theta$. Then, the following bound holds:
    \begin{equation}
        \left\| \frac{1}{t} \sum_{i=1}^{t} (P^{\pi_\theta})^i(s_0, \cdot) - d^{\pi_\theta}(\cdot) \right\|_{\mathrm{TV}} \leq \frac{C}{t}, \quad \forall s_0 \in \mathcal{S}.
    \end{equation}
\end{lemma}

\begin{lemma}[Lemma 6, \citep{ganesh2025regret}]
    \label{lemma: recurrent time average}
    Let Assumption~\ref{assumption: unichain} hold and consider any $\theta \in \Theta$. Then, the following bound holds:
    \begin{equation}
        \left\| \frac{1}{t} \sum_{i=1}^{t} (P^{\pi_\theta})^i(s_0, \cdot) - d^{\pi_\theta}(\cdot) \right\|_{\mathrm{TV}} \leq \frac{C_\mathrm{tar}}{t}, \quad \forall s_0 \in \mathcal{S}_R^\theta.
    \end{equation}
\end{lemma}

\begin{lemma}[\citep{ganesh2025regret}, Bias of Markovian Sample Average]
    Let Assumption~\ref{assumption: unichain} hold, and let $f : \mathcal{S} \to \mathbb{R}^d$ satisfy $\|f(s)\| \le C_f$ for all $s \in \mathcal{S}$, for some constant $C_f > 0$. Then, the following bound holds:
    \[
        \left\| \mathbb{E} \left[ \frac{1}{B} \sum_{i=1}^B f(s_i) \right] - \mu \right\| \le \frac{\sqrt{d} C_f C}{B}
    \]
where $\mathbb{E}$ denotes the expectation of the Markov chain $\{s_i\}$ induced by $\pi_\theta$ starting from any $s_0 \in \mathcal{S}$ and $\mu = \mathbb{E}_{s \sim d^{\pi_\theta}}[f(s)]$. Furthermore, if $s_1 \in \mathcal{S}_R^\theta$, then
\[
\left\| \mathbb{E} \left[ \frac{1}{B} \sum_{i=1}^B f(s_i) \right] - \mu \right\| \le \frac{\sqrt{d} C_f C_{\mathrm{tar}}}{B}.
\]
\end{lemma}

\begin{lemma}[\citep{ganesh2025regret} Variance of Markovian Sample Average]
    Let Assumption 1 hold, and let $f : \mathcal{S} \to \mathbb{R}^d$ satisfy $\|f(s)\| \le C_f$ for all $s \in \mathcal{S}$, for some constant $C_f > 0$. Then, the following bounds hold:
    \[
        \mathbb{E} \left\| \frac{1}{B} \sum_{i=1}^B f(s_i) - \mu \right\|^2 \le \frac{C_f^2 + 2\sqrt{d}C_f^2 C}{B}
    \]
    where $\mathbb{E}$ denotes the expectation of the Markov chain $\{s_i\}$ induced by $\pi_\theta$ starting from any $s_0 \in \mathcal{S}$ and $\mu = \mathbb{E}_{s \sim d^{\pi_\theta}}[f(s)]$. Furthermore, if $s_1 \in \mathcal{S}_R^\theta$, then
    \[
    \mathbb{E} \left\| \frac{1}{B} \sum_{i=1}^B f(s_i) - \mu \right\|^2 \le \frac{C_f^2 + 2\sqrt{d}C_f^2 C_{\mathrm{tar}}}{B}.
    \]
\end{lemma}

\begin{lemma}[MSE of Markovian sample average for Unichain MDP]
Let $(Z_t)_{t \ge 0}$ be a unichain Markov chain with stationary distribution $d_Z$.
Let $g : \mathcal Z \to \mathbb R^d$ satisfy
\[
    \int g(z)\, d_Z(z) = 0, \quad \|g(z)\| \leq \sigma \quad \forall z \in \mathcal Z,
\]
Let \(C_1 = C_{\mathrm{tar}}\) if \(Z_0 \in \mathcal{S}_R^\theta\) and \(C_1 = C\) otherwise. Then for all $N \geq 1$,
\[
\mathbb E \left\| \frac{1}{N} \sum_{t=0}^{N-1} g(Z_t) \right\|^2
\leq \frac{16 C_1^2 \sigma^2}{N}.
\]
\end{lemma}

\begin{proof}
    Fix $z \in \mathcal Z$. Since $\int g \, d_Z = 0$ and $\|g\|_\infty \leq \sigma$, we have
    \begin{align}
        \left\| \sum_{t=0}^{T-1} \mathbb E[ g(Z_t) \mid Z_0 = z ] \right\|
        &= \left\| \sum_{t=0}^{T-1} \int g(u) \left(P^t(z, \mathrm{d}u) - d_Z(\mathrm{d}u) \right)\right\| \\
        &\leq \sum_{t=0}^{T-1} \int |g(u)| \, \left| P^t(z, \mathrm{d}u) - d_Z(\mathrm{d}u)\right|\\
        &\leq \sigma \sum_{t=0}^{T-1} \| P^t(z,\cdot) - d_Z \|_{\mathrm{TV}} \\
        &\leq C_1\sigma
    \end{align}
which holds by either Lemma~\ref{lemma: non recurrent time average} or Lemma~\ref{lemma: recurrent time average} depending on if \(Z_0 \in \mathcal{S}_{R}^\theta\). Define
\begin{align}
    h(z) = \sum_{t=0}^{\infty} \mathbb E[g(Z_t) \mid Z_0 = z ] .
\end{align}
The above bound implies that the series converges absolutely and that
\begin{align}
    \| h \|_\infty \leq C_1 \sigma
\end{align}
Standard unichain theory implies that $h$ satisfies the Poisson equation
\begin{align}
    (I - P) h = g, \quad P h(z) = \mathbb E[ h(Z_{t+1}) \mid Z_t = z ] .
\end{align}
where \(I\) is the identity matrix and \(P\) is the transition kernel. Using the Poisson equation,
\begin{align}
    g(Z_t) = h(Z_t) - \mathbb E[ h(Z_{t+1}) \mid \mathcal F_t ].
\end{align}
where \(\mathcal{F}_t\) is the filtration of the Markov Chain. Summing from $t = 0$ to $N - 1$ gives
\begin{align}
    S_N &= \sum_{t=0}^{N-1} g(Z_t) \\
    &= \sum_{t=0}^{N-1} \left( h(Z_t) - \mathbb E[ h(Z_{t+1}) \mid \mathcal F_t ] \right) \\
    &= h(Z_0) - h(Z_N) + \sum_{t=0}^{N-1} \left( h(Z_{t+1}) - \mathbb E[ h(Z_{t+1}) \mid \mathcal F_t ] \right) .
\end{align}
We define \(B_N = h(Z_0) - h(Z_N)\) and \(M_{t+1} = h(Z_{t+1}) - \mathbb E[ h(Z_{t+1}) \mid \mathcal F_t]\). Then
\begin{align}
    S_N = B_N + \sum_{t=1}^{N} M_t ,
\end{align}
and $(M_t)$ is a martingale difference sequence. Since $\| h \|_\infty \leq C_1 \sigma$, we have
\begin{align}
    \mathbb{E}\left[\left\| B_N \right\|^2\right] \leq \mathbb{E}\left[\left(\| h(Z_0) \| + \| h(Z_N) \|\right)^2\right] \leq 4 C_1^2 \sigma^2 .
\end{align}
Moreover,
\begin{align}
\mathbb{E}\left[\| M_t \|^2\right] &\leq \mathbb{E} \left[\left(\| h(Z_t) \| + \| \mathbb E[ h(Z_t) \mid \mathcal F_{t-1} ] \|\right)^2\right] \leq 4\| h \|_\infty^2 \leq 4 C_1^2 \sigma^2
\end{align}
Since martingale differences are orthogonal,
\begin{align}
\mathbb E \left\| \sum_{t=1}^{N} M_t \right\|^2
= \sum_{t=1}^{N} \mathbb E \| M_t \|^2
\leq 4 N C_1^2 \sigma^2 .
\end{align}

Using $\|a + b\|^2 \leq 2 \|a\|^2 + 2 \|b\|^2$,
\begin{align}
    \mathbb E \| S_N \|^2 \leq 2 \mathbb E \| B_N \|^2 + 2 \mathbb E \left\| \sum_{t=1}^{N} M_t \right\|^2 \leq 16 N C_1^2 \sigma^2 .
\end{align}
Dividing by $N^2$ gives
\begin{align}
    \mathbb E \left\| \frac{1}{N} \sum_{t=0}^{N-1} g(Z_t) \right\|^2
\leq \frac{16 C_1^2 \sigma^2}{N} 
\end{align}
\end{proof}

\begin{lemma}
    Consider a time-homogeneous, unichain Markov chain \((Z_t)_{t \geq 0}\) with a unique stationary distribution \(d_Z\) and a mixing time \(t_{\mathrm{mix}}\). Let \(C_1 = C_{\mathrm{tar}}\) if \(Z_0 \in \mathcal{S}_R^\theta\) and \(C_1 = C\) otherwise. Assume that \(\nabla F(x)\) is an arbitrary gradient and \(\nabla F(x, Z)\) denotes an estimate of \(\nabla F(x)\). Let \(\|\mathbb{E}_{d_Z} \left[\nabla F(x, Z)\right] - \nabla F(x)\|^2 \leq \delta^2\) and \(\|\mathbb{E}_{d_Z} \left[\nabla F(x, Z)\right] - \nabla F(x, Z_t)\|^2 \leq \sigma^2\) for all \(t \geq 0\). If \(Q \sim \mathrm{Geom}\left(\frac{1}{2}\right)\), then the following MLMC estimator
    \begin{align}
        g_{\mathrm{MLMC}} = g^0 + \begin{cases}
            2^Q \left(g^Q - g^{Q-1}\right) & \text{if } 2^{Q} \leq T_{\mathrm{max}} \\
            0 & \text{otherwise}
        \end{cases} \quad \text{where } g^{j} = \frac{1}{2^j} \sum_{t=1}^{2^j} \nabla F(x, Z_t)
    \end{align}
    satisfies the following inequalities:
    \begin{enumerate}
        \item \(\mathbb{E}\left[g_{\mathrm{MLMC}}\right] = \mathbb{E}\left[g^{\left\lfloor \log T_{\mathrm{max}} \right\rfloor}\right]\)
        \item \(\mathbb{E} \left[\| \nabla F(x) - g_{\mathrm{MLMC}}\|^2\right] \leq \mathcal{O}(\sigma^2 C_1^2 \log_2 T_{\mathrm{max}} + \delta^2)\)
        \item \(\|\nabla F(x) - \mathbb{E}\left[g_{\mathrm{MLMC}}\right]\|^2 \leq \mathcal{O}(\sigma^2 C_1^2 T_{\mathrm{max}}^{-1} + \delta^2)\)
    \end{enumerate}
\end{lemma}

\begin{proof}
    We begin by proving statement 1:
    \begin{align}
        \mathbb{E}\left[g_{\mathrm{MLMC}}\right] &= \mathbb{E} \left[g^0 + \sum_{j=1}^{\lfloor \log_2 T_{\mathrm{max}} \rfloor} \Pr[Q = j] 2^j \left(g^j - g^{j-1}\right)\right] \\
        &= \mathbb{E} \left[g^0 + \sum_{j=1}^{\lfloor \log_2 T_{\mathrm{max}} \rfloor} \left(g^j - g^{j-1}\right)\right] \\
        &= \mathbb{E} \left[g^0 + \left(g^{\lfloor \log_2 T_{\mathrm{max}} \rfloor} - g^{0}\right)\right] \\
        &= \mathbb{E} \left[g^{\lfloor \log_2 T_{\mathrm{max}} \rfloor}\right]
    \end{align}
    Then for statement 2, we have
    \begin{align}
        &\mathbb{E} \left[\|\mathbb{E}_{d_Z} \left[\nabla F(x, Z)\right] - g_{\mathrm{MLMC}}\|^2\right] \\
        &\leq 2 \mathbb{E} \left[\|\mathbb{E}_{d_Z} \left[\nabla F(x, Z)\right] - g^0\|^2\right] + 2\mathbb{E} \left[\|g_{\mathrm{MLMC}} - g^0\|^2\right] \\
        &= 2 \mathbb{E} \left[\|\mathbb{E}_{d_Z} \left[\nabla F(x, Z)\right] - g^0\|^2\right] + 2\sum_{j=1}^{\lfloor \log_2 T_{\mathrm{max}} \rfloor} 2^j \mathbb{E} \left[\|g^j - g^{j-1}\|^2\right] \\
        &\leq 2 \mathbb{E} \left[\|\mathbb{E}_{d_Z} \left[\nabla F(x, Z)\right] - g^0\|^2\right] \\
        &\qquad +4\sum_{j=1}^{\lfloor \log_2 T_{\mathrm{max}} \rfloor} 2^j \left(\mathbb{E} \left[\|\mathbb{E}_{d_Z}\left[\nabla F(x, Z)\right] - g^{j-1}\|^2\right] + \mathbb{E} \left[\|\mathbb{E}_{d_Z}\left[\nabla F(x, Z)\right] - g^{j}\|^2\right]\right) \\
        &\stackrel{(a)}{\leq} 16C_1^2 \sigma^2 \left[2 + 4\sum_{j=1}^{\lfloor \log_2 T_{\mathrm{max}} \rfloor} 2^j \left(\frac{1}{2^j} + \frac{1}{2^{j-1}}\right)\right] \\
        &\leq \mathcal{O}(C_1^2 \sigma^2 \log_2 T_{\mathrm{max}})
    \end{align}
    where \((a)\) uses Lemma~\ref{lemma: mse unichain markov chain}. Then we have
    \begin{align}
        \mathbb{E} \left[\| \nabla F(x) - g_{\mathrm{MLMC}}\|^2\right] &\leq 2\mathbb{E} \left[\|\mathbb{E}_{d_Z} \left[\nabla F(x, Z)\right] - g_{\mathrm{MLMC}}\|^2\right] + 2\mathbb{E} \left[\|\mathbb{E}_{d_Z} \left[\nabla F(x, Z)\right] - \nabla F(x)\|^2\right] \\
        &\leq \mathcal{O}(C_1^2 \sigma^2 \log_2 T_{\mathrm{max}} + \delta^2)
    \end{align}
    which proves statement 2. Finally, for statement 3, we have
    \begin{align}
        \| \nabla F(x) - \mathbb{E}\left[g_{\mathrm{MLMC}}\right]\|^2 &\leq 2 \|\mathbb{E}_{d_Z} \left[\nabla F(x, Z)\right] - g_{\mathrm{MLMC}}\|^2 + 2\|\mathbb{E}_{d_Z} \left[\nabla F(x, Z)\right] - \nabla F(x)\|^2 \\
        &\leq 2\delta^2 + 2 \left\|\mathbb{E}_{d_Z} \left[\nabla F(x, Z)\right] - g^{\left\lfloor \log_2 T_{\mathrm{max}} \right\rfloor}\right\|^2 \\
        &\leq \mathcal{O}(\sigma^2 C_1^2 T_{\mathrm{max}}^{-1} + \delta^2)
    \end{align}
\end{proof}

%% file: Sections/Appendix/critic_analysis.tex
\section{Critic Analysis}
\label{sec: critic analysis}
Recall that $\mathbf{A}_g(\theta) = \mathbb{E}_\theta[\mathbf{A}_g(\theta; z)]$ and $\mathbf{b}_g(\theta) = \mathbb{E}_\theta[\mathbf{b}_g(\theta; z)]$, where $\mathbb{E}_\theta$ denotes the expectation over the distribution of $z = (s, a, s')$ with $(s, a) \sim \nu^{\pi_\theta}$ and $s' \sim P(\cdot|s, a)$. For a large enough constant $c_\gamma$, Assumption~\ref{assumption: critic_matrix} implies that the critic matrix $\mathbf{A}_g(\theta)$ satisfies the following lower bound on its quadratic form: $$ \xi^\top \mathbf{A}_g(\theta)\xi \ge \frac{\lambda}{2} \|\xi\|^2 $$ for both $g \in \{r, c\}$ and all $\theta \in \Theta$, for all vectors $\xi = [\eta, \zeta]^\top$ such that $\zeta \in \text{ker}(M_\theta)^\perp$.

\begin{lemma}[Subspace Positive Definiteness for Unichain MDPs]
\label{lemma: subspace positive definiteness}
For a large enough $c_\gamma$, Assumption 4.2 implies that $\xi^\top \mathbf{A}_g(\theta)\xi \geq \frac{\lambda}{2} \|\xi\|^2$ for all $\theta \in \Theta$ and for all $\xi = [\eta, \zeta]^\top$ such that $\zeta \in \text{ker}(M_\theta)^\perp$.
\end{lemma}
\begin{proof}

Note that for any $\xi = [\eta, \zeta]^\top$, the expansion of the quadratic form for the critic matrix is given by:
$$ \xi^\top \mathbf{A}_g(\theta)\xi = c_\gamma \eta^2 + \eta \zeta^\top \mathbb{E}_{\theta} [\phi_g(s)] + \zeta^\top \mathbb{E}_{\theta} \left[ \phi_g(s) [\phi_g(s) - \phi_g(s')]^\top \right] \zeta $$
Using Assumption 4.2 (which provides the lower bound $\lambda$ on the recurrent block for $\zeta \perp \text{ker}(M_\theta)$) and the fact that $\|\phi_g(s)\| \le 1$ for all $s \in \mathcal{S}$, we have:
\begin{align}
\xi^\top \mathbf{A}_g(\theta)\xi &\stackrel{(a)}{\ge} c_\gamma \eta^2 - |\eta| \|\zeta\| + \lambda \|\zeta\|^2 \\
&\ge \|\xi\|^2 \left\{ \min_{u \in [0,1]} c_\gamma u - \sqrt{u(1-u)} + \lambda(1-u) \right\} \\
&\stackrel{(b)}{\ge} \frac{\lambda}{2} \|\xi\|^2
\end{align}
where $u = \frac{\eta^2}{\eta^2 + \|\zeta\|^2}$. Here, $(a)$ is a direct consequence of the subspace property of the feature mapping and the bounded rewards/features. Finally, the inequality $(b)$ is satisfied when:
$$ c_\gamma \ge \lambda + \sqrt{\frac{1}{\lambda^2} - 1} $$
\end{proof}

\begin{lemma}[Kernel Inclusion for Unichain CMDPs]
\label{lemma: kernel inclusion}
Consider a constrained MDP with a critic matrix $\mathbf{A}_g(\theta)$, \(g \in \{r, c\}\). Let $s \in \mathcal{S}_R^\theta$ be a state in the recurrent class and $s' \sim P^{\pi_\theta}(\cdot|s)$ be the subsequent state. Then:
\begin{equation}
    \text{ker}(\mathbf{A}_g(\theta)) \subseteq \text{ker}(\mathbf{A}_g(\theta, z)), \quad \forall g \in \{r, c\}
\end{equation}
where $z = (s, a, s')$.
\end{lemma}

\begin{proof}
    Arbitrarily pick a \(z \in \text{ker}(\mathbf{A}_g(\theta))\). Then we have \(z = [0, \zeta^\top]^\top\), with \(\zeta \in \text{ker}(M_\theta^g)\). Thus,
    \begin{align}
        \mathbf{A}_g(\theta)z = \begin{bmatrix}
            0 \\ \phi^g(s)(\phi^g(s) - \phi^g(s'))^\top \zeta
        \end{bmatrix} = \begin{bmatrix}
            0 \\ \phi^g(s)(\phi^g(s)^\top \zeta - \phi^g(s')^\top \zeta)
        \end{bmatrix} = \begin{bmatrix}
            0 \\ 0
        \end{bmatrix}
    \end{align}
    as both \(s\) and \(s'\) are in \(\mathcal{S}_{R}^\theta\). Since we picked \(z\) arbitrarily, we have
    \begin{align}
        \text{ker}(\mathbf{A}_g(\theta)) \subseteq \text{ker}(\mathbf{A}_g(\theta, z)), \quad \forall g \in \{r, c\}
    \end{align}
\end{proof}

We note that given event \(\mathcal{E}_B\), we have that all states after the burn in period are in the recurrent class \(\mathcal{S}_R^{\theta_k}\), which allows us to use Lemma~\ref{lemma: kernel inclusion}.

\begin{lemma}[Unichain MLMC Critic Error Bounds]
\label{lemma: unichain-mlmc-critic}
    Fix an outer iteration $k$, a critic index $g \in \{r,c\}$, and let $\mathbf{A}_g(\theta_k)$ and $\mathbf{b}_g(\theta_k)$ denote the exact stationary critic matrix and vector.
    Let $\mathbf{A}^{\mathrm{MLMC}}_{g,kh}$ and $\mathbf{b}^{\mathrm{MLMC}}_{g,kh}$ be the MLMC estimators defined by
    \begin{align}
        \mathbf{A}_{g, kh}^{\mathrm{MLMC}}
        &=
        \mathbf{A}_{g, kh}^0 + \mathbb{I}\{2^Q \le T_{\max}\} \cdot 2^Q \left(\mathbf{A}_{g, kh}^Q - \mathbf{A}_{g, kh}^{Q-1}
        \right), \\
        \mathbf{b}_{g, kh}^{\mathrm{MLMC}} &= \mathbf{b}_{g, kh}^0 + \mathbb{I}\{2^Q \le T_{\max}\} \cdot 2^Q \left(\mathbf{b}_{g, kh}^Q - \mathbf{b}_{g, kh}^{Q-1} \right),
    \end{align}
    where $Q \sim \mathrm{Geom}(1/2)$ and
    \begin{align}
        \mathbf{A}_{g,kh}^j = \frac{1}{2^j} \sum_{t=1}^{2^j} \mathbf{A}_g(\theta_k; z_t), \qquad \mathbf{b}_{g,kh}^j = \frac{1}{2^j} \sum_{t=1}^{2^j} \mathbf{b}_g(\theta_k; z_t),
    \end{align}
    with $\{z_t\}_{t \ge 1}$ generated by the unichain Markov chain under $\pi_{\theta_k}$. Then the following bounds hold.
    \begin{align}
        \mathbb{E}_{k,h} \left[\left\|        \mathbf{A}^{\mathrm{MLMC}}_{g,kh} - \mathbf{A}_g(\theta_k) \right\|^2 \right] &\leq \mathcal{O}\left(
        c_\gamma^2 C_{\mathrm{tar}}^2 \log T_{\max} \right), \\ 
        \mathbb{E}_{k,h} \left[ \left\| \mathbf{b}^{\mathrm{MLMC}}_{g,kh} - \mathbf{b}_g(\theta_k) \right\|^2 \right] &\leq \mathcal{O}\left( c_\gamma^2 C_{\mathrm{tar}}^2 \log T_{\max} \right) \\
        \left\|        \mathbb{E}_{k,h}\left[\mathbf{A}^{\mathrm{MLMC}}_{g,kh}\right] - \mathbf{A}_g(\theta_k) \right\|^2 &\le \mathcal{O}\left( c_\gamma^2 C_{\mathrm{tar}}^2T_{\max}^{-1} \right), \\
        \left\| \mathbb{E}_{k,h} \left[\mathbf{b}^{\mathrm{MLMC}}_{g,kh}\right] - \mathbf{b}_g(\theta_k) \right\|^2 &\le \mathcal{O}\left( c_\gamma^2 C_{\mathrm{tar}}^2T_{\max}^{-1} \right)
    \end{align}
    where \(\mathbb{E}_{kh}\) is the conditional expectation given the history for the current inner loop for outer loop iteration \(k\).
\end{lemma}

\begin{proof}
    We have that the subspace positive definiteness (C8) and kernel inclusion assumptions (C9) are satisfied by Lemmas~\ref{lemma: subspace positive definiteness} and~\ref{lemma: kernel inclusion}. Then for any transition \(z = (s, a, s')\), we have
    \begin{align}
        \|\mathbf{A}_g(\theta_k; z)\| &\le |c_\gamma| + \|\phi_g(s)\| + \|\phi_g(s) (\phi_g(s) - \phi_g(s'))^\top\| \leq c_\gamma + 3 = \mathcal{O}(c_\gamma), \\
        \|\mathbf{b}_g(\theta_k; z)\| &\le |c_\gamma g(s,a)| + \|g(s,a) \phi_g(s)\| \leq c_\gamma + 1 = \mathcal{O}(c_\gamma),
    \end{align}
    since \(|g(s, a)| \leq 1\) and \(\|\phi(s)\| \leq 1\) \(\forall s, a \in \mathcal{S} \times \mathcal{A}\). Then for any \(z_{kh}^j\) with \(h \geq B\), we have
    \begin{align}
        \|\mathbf{A}_g(\theta_k; z_{kh}^j) - \mathbf{A}_g(\theta_k)\|^2 &\le \mathcal{O}(c_\gamma^2), \\
        \|\mathbf{b}_g(\theta_k; z_{kh}^j) - \mathbf{b}_g(\theta_k)\|^2 &\le \mathcal{O}(c_\gamma^2)
    \end{align}
    Then setting \(\sigma^2 = \mathcal{O}(c_\gamma^2)\) and \(\delta^2 = 0\) due to the unbiased estimate, we obtain the result.
\end{proof}

\begin{theorem}[Unichain Critic Convergence with MLMC]
    Consider the critic estimation subroutine in Algorithm 1 for a unichain CMDP. Suppose the assumptions of Theorem~\ref{theorem: unichain stochastic linear recursion} hold. Then conditioned on the event $\mathcal{E}_B$, the iterates satisfy:
    \begin{equation}
        \mathbb{E}_k [\|\Pi(\xi_{g, H}^k - \xi^*_g(\theta_k))\|^2 \mid \mathcal{E}_B] \leq \tilde{\mathcal{O}} \left(\frac{c_\gamma^2}{\lambda^2 T^2} + \frac{c_\gamma^4 C_{\mathrm{tar}}^2}{\lambda^4} \left(\frac{1}{H} + \frac{1}{T_{\mathrm{max}}}\right)\right)
    \end{equation}
    \begin{equation}
        \|\Pi(\mathbb{E}_k[\xi_{g, H}^k  \mid  \mathcal{E}_B] - \xi^*_g(\theta_k))\|^2 \leq \tilde{\mathcal{O}} \left( \frac{c_\gamma^2}{T^2 \lambda^2} + \frac{c_\gamma^4 C_{\mathrm{tar}}^2}{\lambda^4 T_{\mathrm{max}}}\right)
    \end{equation}
\end{theorem}
\begin{proof}
    From the bounds on \(\|\mathbf{A}_g(\theta_k; z)\|\) and \(\|\mathbf{b}_g(\theta_k; z)\|\), we satisfy the norm bounds on \(P\) (C6) and \(q\) (C7). Additionally, Lemma~\ref{lemma: unichain-mlmc-critic}'s result satisfies (C1)-(C4). Finally, we set \(\bar{\delta}_q^2 = \delta_q^2\) to satisfy (C5). Then if we are given that the initial state is within the recurrent class, then for \(c_\gamma = \lambda + \sqrt{\frac{1}{\lambda^2} - 1}\) and \(\gamma_{\xi} = \frac{8 \log T}{\lambda H}\) with \(T_{\max} \geq \frac{8 c_\gamma^2 C_{\mathrm{tar}}^2}{\lambda}\) and \(\gamma_\xi \leq \frac{\lambda}{24 c_\gamma^2 C_{\mathrm{tar}}^2\log T_{\mathrm{max}}}\), using Theorem~\ref{theorem: unichain stochastic linear recursion} yields
    \begin{align}
        \mathbb{E}_k [\|\Pi(\xi_{g, H}^k - \xi^*_g(\theta_k))\|^2] &\leq \tilde{\mathcal{O}} \left( \frac{1}{T^2} \|\Pi(\xi_0 - \xi_g^*(\theta_k))\|^2 + \frac{\log T}{H} \lambda^{-2}\left( c_\gamma^4 C_\mathrm{tar}^2 \lambda^{-2} + c_\gamma^2 C_\mathrm{tar}^2\right) + \lambda^{-2} \left(\delta_P^2 \lambda^{-2} c_\gamma^2 + \delta_q^2\right) \right) \nonumber \\
        &\leq \tilde{\mathcal{O}} \left(\frac{1}{T^2} \|\Pi(\xi_0 - \xi_g^*(\theta_k))\|^2  + \lambda^{-4} c_\gamma^4 C_{\mathrm{tar}}^2 \left(\frac{1}{H} + \frac{1}{T_{\mathrm{max}}}\right)\right)
    \end{align}
    \begin{align}
        \|\Pi(\mathbb{E}_k[\xi_{g, H}^k] - \xi^*_g(\theta_k))\|^2 \leq \tilde{\mathcal{O}} \left( \left(\frac{1}{T^2} + \frac{c_\gamma^2 C_\mathrm{tar}^2}{\lambda^2 T_\mathrm{max}}\right) \|\Pi(\xi_0 - \xi_g^*(\theta_k))\|^2 + \frac{c_\gamma^4 C_\mathrm{tar}^2}{\lambda^4 T_\mathrm{max}}\right)
    \end{align}
    From Lemma~\ref{lemma: subspace positive definiteness}, we have \(\|\mathbf{A}_g(\theta_k)\| \|\xi\|^2 \geq \frac{\lambda}{2} \|\xi\|^2\) for all \(\xi\), so \(\|\mathbf{A}_g(\theta_k)\| \geq \frac{\lambda}{2}\). Then \(\|\xi_g^*(\theta_k)\|^2 = \|\mathbf{A}_g(\theta_k)^{-1} \mathbf{b}_g(\theta_k)\|^2 \leq \tilde{\mathcal{O}}\left(\frac{ c_\gamma^2}{\lambda^2}\right)\). Thus,
    \begin{align}
        \mathbb{E}_k [\|\Pi(\xi_{g, H}^k - \xi^*_g(\theta_k))\|^2] &\leq \tilde{\mathcal{O}} \left(\frac{c_\gamma^2}{\lambda^2 T^2} + \lambda^{-4} c_\gamma^4 C_{\mathrm{tar}}^2 \left(\frac{1}{H} + \frac{1}{T_{\mathrm{max}}}\right)\right) \\
        \|\Pi(\mathbb{E}_k[\xi_{g, H}^k] - \xi^*_g(\theta_k))\|^2 &\leq \tilde{\mathcal{O}} \left( \left(\frac{1}{T^2} + \frac{c_\gamma^2 C_\mathrm{tar}^2}{\lambda^2 T_\mathrm{max}}\right) \frac{c_\gamma^2}{\lambda^2} + \frac{c_\gamma^4 C_\mathrm{tar}^2}{\lambda^4 T_\mathrm{max}}\right) \nonumber \\
        &= \tilde{\mathcal{O}} \left( \frac{c_\gamma^2}{T^2 \lambda^2} + \frac{c_\gamma^4 C_\mathrm{tar}^2}{\lambda^4 T_\mathrm{max}}\right)
    \end{align}
    as \(\xi_0 = 0\) and \(\xi_{g, H}^k = \xi_g^k\).
\end{proof}

%% file: Sections/Appendix/npg_analysis.tex
\section{NPG Analysis}
\label{sec: npg analysis}
In this section, we provide the detailed analysis for the Natural Policy Gradient (NPG) estimation subroutine in Algorithm~\ref{alg: pdnac_bi}. The NPG analysis builds upon the critic results from the previous section and the general linear recursion framework established in Appendix~\ref{sec: unichain stochastic linear recursions}. We note that for this section, we have that the initial state is within the recurrent class since we are given \(\mathcal{E}_B\).

\subsection{MLMC Estimation Parameters for Fisher and Policy Gradient}
\label{subsec: fisher and pg mlmc estimation parameters}
To apply Lemma~\ref{lemma: mlmc statements}, we need to characterize the variance parameter $\sigma^2$ and bias parameter $\delta^2$ for both the Fisher matrix estimates and the policy gradient estimates. For the Fisher matrix component $\hat{F}(\theta_k; z)$, we have:
\begin{equation}
    \mathbb{E}_{\theta_k}[\hat{F}(\theta_k; z)] = F(\theta_k),
\end{equation}
so the estimate is unbiased: $\delta^2_F = 0$. For the variance, using Assumption~\ref{assumption: score bounds}:
\begin{align}
    \|\hat{F}(\theta_k; z) - F(\theta_k)\|^2 &\leq \|\hat{F}(\theta_k; z)\|^2 + \|F(\theta_k)\|^2\nonumber\\
    &\leq 2\|\hat{F}(\theta_k; z)\|^2\nonumber\\
    &= 2\|\nabla_\theta \log \pi_{\theta_k}(a|s)\|^4\nonumber\\
    &\leq 2G_1^4.
\end{align}
Therefore, the MLMC variance parameter for the Fisher matrix is:
\begin{equation}
    \sigma^2_F = 2G_1^4, \quad \delta^2_F = 0.
\end{equation}
For the policy gradient component $\hat{\nabla}_\theta J_g(\theta_k, \xi^k_g; z)$, the analysis is more involved due to the dependence on the critic estimate $\xi^k_g$. First, we bound the variance. For a single transition $z = (s,a,s')$, we have:
\begin{align}
\|\hat{\nabla}_\theta J_g(\theta_k, \xi^k_g; z)\|^2 &= \hat{A}^{\pi_{\theta_k}}_g(\xi^k_g; z)^2 \|\nabla_\theta \log \pi_{\theta_k}(a|s)\|^2\nonumber\\
&\leq G_1^2 \cdot \hat{A}^{\pi_{\theta_k}}_g(\xi^k_g; z)^2.
\end{align}
Using the decomposition for the advantage function, for $s, s' \in \mathcal{S}^{\theta_k}_R$
\begin{align}
    |\hat{A}^{\pi_{\theta_k}}_g(\xi^k_g; z)| &\leq |g(s,a)| + |\eta^k_g| + |\langle \phi_g(s') - \phi_g(s), \Pi \zeta^k_g \rangle|\nonumber\\
    &\leq 1 + |\eta^k_g| + 2\|\Pi \zeta^k_g\|\nonumber\\
    &\leq \mathcal{O}(1 + \|\Pi \xi^k_g\|),
\end{align}
Therefore
\begin{equation}
    \mathbb{E}_{\theta_k}[\hat{A}^{\pi_{\theta_k}}_g(\xi^k_g; z)^2] \leq \mathcal{O}((1 + \|\Pi \xi^k_g\|)^2) \leq \mathcal{O}(\|\Pi \xi^k_g\|^2),
\end{equation}
which gives the variance parameters
\begin{equation}
    \sigma^2_{g,k} := \mathbb{E}_{\theta_k}[\|\hat{\nabla}_\theta J_g(\theta_k, \xi^k_g; z)\|^2] \leq \mathcal{O}(G_1^2 \|\Pi \xi^k_g\|^2), \quad \bar{\sigma}^2_{g,k} := \mathbb{E}_k[\sigma^2_{g,k}] \leq \mathcal{O}(G_1^2 \mathbb{E}_k[\|\Pi \xi^k_g\|^2])
\end{equation}
We will provide the bounds for \(\delta_{g, k}^2\) and \(\bar{\delta}_{g, k}^2\) in the following section.

\subsection{Verification of Linear Recursion Conditions}

We verify that the NPG update satisfies conditions (C1)-(C9) of Theorem~\ref{theorem: unichain stochastic linear recursion}, conditioned on the event $\mathcal{E}_B$ that ensures all samples are collected from the recurrent class $\mathcal{S}^{\theta_k}_R$. Note that $F(\theta_k)$ is the Fisher information matrix, which under Assumption~\ref{assumption: fisher} satisfies $F(\theta_k) \succeq \mu I_d$ for all $\theta_k$. Therefore, $\ker(F(\theta_k)) = \{0\}$, and condition (C9) holds trivially. From Assumption~\ref{assumption: fisher}, we have $x^\top F(\theta_k) x \geq \mu \|x\|^2$ for all $x \in \mathbb{R}^d$. Thus, condition (C8) is satisfied with $\lambda_P = \mu$. From Assumption~\ref{assumption: score bounds}, for any $z = (s,a,s') \in \mathcal{S} \times \mathcal{A} \times \mathcal{S}$, we have:
\begin{align}
    \|F(\theta_k; z)\| = \|\nabla_\theta \log \pi_{\theta_k}(a|s) \otimes \nabla_\theta \log \pi_{\theta_k}(a|s)\| \leq \|\nabla_\theta \log \pi_{\theta_k}(a|s)\|^2 \leq G_1^2.
\end{align}

Applying Lemma~\ref{lemma: mlmc statements} with the results from \ref{subsec: fisher and pg mlmc estimation parameters}, we obtain:
\begin{align}
    \mathbb{E}_{k,h} \left[\|F^{\text{MLMC}}_{kh} - F(\theta_k)\|^2\right] &\leq \mathcal{O}\left(G_1^4 C_{\mathrm{tar}}^2 \log^2 T_{\max}\right) \\
    \|\mathbb{E}_{k,h}[F^{\text{MLMC}}_{kh}] - F(\theta_k)\|^2 &\leq \mathcal{O}\left(G_1^4 C_{\mathrm{tar}}^2 T_{\max}^{-1}\right),
\end{align}
where $\mathbb{E}_{k,h}$ denotes conditional expectation given the history up to inner iteration $h$ of the NPG subroutine at outer iteration $k$. From Lemma 7 in \citet{ganesh2025regret}, we have $|A^{\pi_{\theta_k}}_g(s,a)| \leq 1 + 4(C_{\mathrm{hit}} + C_{\mathrm{tar}}) =: \mathcal{O}(C)$ for all $(s,a)$. Combined with Assumption~\ref{assumption: score bounds}, this yields:
\begin{equation}
    \|\nabla_\theta J_g(\theta_k)\| = \left\|\mathbb{E}_{(s,a) \sim \nu_{\pi_{\theta_k}}}[A^{\pi_{\theta_k}}_g(s,a) \nabla_\theta \log \pi_{\theta_k}(a|s)]\right\| \leq \mathcal{O}(CG_1).
\end{equation}

To analyze the policy gradient estimates, we decompose the error as follows. For arbitrary $\theta_k, \xi^k_g = [\eta^k_g, (\zeta^k_g)^\top]^\top$:
\begin{align}
&\mathbb{E}_{\theta_k}\left[\hat{\nabla}_\theta J_g(\theta_k, \xi^k_g; z)\right] - \nabla_\theta J_g(\theta_k) \nonumber \\
&= \mathbb{E}_{\theta_k}\left[\left(g(s,a) - \eta^k_g + \langle \phi_g(s') - \phi_g(s), \zeta^k_g \rangle\right) \nabla_\theta \log \pi_{\theta_k}(a|s)\right] - \nabla_\theta J_g(\theta_k)\nonumber\\
&\stackrel{(a)}{=} \mathbb{E}_{\theta_k}\left[\left(\eta^*_g(\theta_k) - \eta^k_g + \langle \phi_g(s') - \phi_g(s), \zeta^k_g - \zeta^*_g(\theta_k) \rangle\right) \nabla_\theta \log \pi_{\theta_k}(a|s)\right] \nonumber\\
&\quad + \mathbb{E}_{\theta_k}\left[\left(\langle \phi_g(s), \zeta^*_g(\theta_k) \rangle - V^{\pi_{\theta_k}}_g(s) + V^{\pi_{\theta_k}}_g(s') - \langle \phi_g(s'), \zeta^*_g(\theta_k) \rangle\right) \nabla_\theta \log \pi_{\theta_k}(a|s)\right] \nonumber\\
&\quad + \mathbb{E}_{\theta_k}\left[\left(g(s,a) - \eta^*_g(\theta_k) + V^{\pi_{\theta_k}}_g(s) - V^{\pi_{\theta_k}}_g(s')\right) \nabla_\theta \log \pi_{\theta_k}(a|s)\right] - \nabla_\theta J_g(\theta_k),
\end{align}
where (a) uses the fact that $\eta^*_g(\theta_k) = J_g(\theta_k)$. The third term vanishes as a consequence of Bellman's equation
\begin{align}
    \mathbb{E}_{\theta_k}& \Big[\left(g(s,a) - \eta^*_g(\theta_k) + V^{\pi_{\theta_k}}_g(s) - V^{\pi_{\theta_k}}_g(s')\right) \nabla_\theta \log \pi_{\theta_k}(a|s)\Big] - \nabla_\theta J_g(\theta_k) \nonumber \\
    &= \mathbb{E}_{\theta_k}\left[\left(Q_g^{\pi_{\theta_k}}(s, a) - V^{\pi_{\theta_k}}_g(s)\right) \nabla_\theta \log \pi_{\theta_k}(a|s)\right] - \nabla_\theta J_g(\theta_k) \nonumber \\
    &= \mathbb{E}_{\theta_k}\left[A_g^{\pi_{\theta_k}}(s, a) \nabla_\theta \log \pi_{\theta_k}(a|s)\right] - \nabla_\theta J_g(\theta_k) = 0 \nonumber
\end{align}
The second term satisfies:
\begin{equation}
\left\|\mathbb{E}_{\theta_k}\left[\left(\langle \phi_g(s), \zeta^*_g(\theta_k) \rangle - V^{\pi_{\theta_k}}_g(s) + V^{\pi_{\theta_k}}_g(s') - \langle \phi_g(s'), \zeta^*_g(\theta_k) \rangle\right) \nabla_\theta \log \pi_{\theta_k}(a|s)\right]\right\|^2 \leq 4G_1^2 \epsilon_{\mathrm{app}},
\end{equation}
by Definition~\ref{def:critic_approx} and Assumption~\ref{assumption: score bounds}. For the first term, note that:
\begin{align}
&\mathbb{E}_{\theta_k}\left[\left(\eta^*_g(\theta_k) - \eta^k_g + \langle \phi_g(s') - \phi_g(s), \zeta^k_g - \zeta^*_g(\theta_k) \rangle\right) \nabla_\theta \log \pi_{\theta_k}(a|s)\right] \nonumber\\
&= \mathbb{E}_{\theta_k}\left[\left(\eta^*_g(\theta_k) - \eta^k_g + \langle \phi_g(s') - \phi_g(s), \Pi(\zeta^k_g - \zeta^*_g(\theta_k)) \rangle\right) \nabla_\theta \log \pi_{\theta_k}(a|s)\right],
\end{align}
where $\Pi$ denotes the projection onto $\ker(M_{\theta_k})^\perp$ and the second equality follows because $\langle \phi_g(s') - \phi_g(s), \Pi^\perp(\zeta^k_g - \zeta^*_g(\theta_k)) \rangle = 0$ for $s, s' \in \mathcal{S}^{\theta_k}_R$ by Lemma 13 in \citet{ganesh2025regret}. Therefore:
\begin{equation}
    \left\|\mathbb{E}_{\theta_k}\left[\hat{\nabla}_\theta J_g(\theta_k, \xi^k_g; z)\right] - \nabla_\theta J_g(\theta_k)\right\|^2 \leq \delta_{g, k}^2 = \mathcal{O}\left(G_1^2 \|\Pi(\xi^k_g - \xi^*_g(\theta_k))\|^2 + G_1^2 \epsilon_{\mathrm{app}}\right).
\end{equation}
Similarly, using Lemma~\ref{lemma: mlmc statements}, we obtain:
\begin{equation}
    \mathbb{E}_{k,h}\left[\left \|\hat{\nabla}_\theta J^{\text{MLMC}}_{g,kh} - \nabla_\theta J_g(\theta_k)\right\|^2\right] \leq \mathcal{O}\left(\sigma^2_{k,g} C_{\mathrm{tar}}^2 \log^2 T_{\max} + G_1^2 \mathbb{E}_k[\|\Pi(\xi^k_g - \xi^*_g(\theta_k))\|^2] + G_1^2 \epsilon_{\mathrm{app}}\right),
\end{equation}
where $\sigma^2_{k,g} = \mathcal{O}(G_1^2 \| \Pi\xi^k_g\|^2)$ from \ref{subsec: fisher and pg mlmc estimation parameters}. This verifies (C3). For condition (C4), also use Lemma~\ref{lemma: mlmc statements}, yielding
\begin{equation}
    \left\|\mathbb{E}_{k,h}[\hat{\nabla}_\theta J^{\mathrm{MLMC}}_{g,kh}] - \nabla_\theta J_g(\theta_k)\right\|^2 \leq \mathcal{O}\left(\frac{C^2_{\mathrm{tar}} G_1^2 \|\Pi\xi^k_g\|^2}{T_{\max}} + G_1^2 \|\Pi(\xi^k_g - \xi^*_g(\theta_k))\|^2 + G_1^2 \epsilon_{\mathrm{app}}\right).
\end{equation}
Finally, for the sharper bound required in condition (C5), we use the full expectation $\mathbb{E}$ instead of $\mathbb{E}_{\theta_k}$. Let $\mathbb{E}_{\theta_k}$ denote the expectation over a single transition $z = (s,a,s')$ where $s \sim d_{\pi_{\theta_k}}$, $a \sim \pi_{\theta_k}(\cdot|s)$, $s' \sim P(\cdot|s,a)$, conditioned on the entire history prior to this transition. Let $\mathbb{E}_k$ denote the expectation over the entire history up to outer iteration $k$, including the randomness in the critic estimate $\xi^k_g$. Then:
\begin{align}
&\mathbb{E}_{\theta_k}\left[\mathbb{E}_{k}\left[\hat{\nabla}_\theta J_g(\theta_k, \xi^k_g; z)\right]\right] - \nabla_\theta J_g(\theta_k) \nonumber\\
&= \mathbb{E}_{\theta_k}\left[\mathbb{E}_{k}\left[\left(g(s,a) - \eta^k_g + \langle \phi_g(s') - \phi_g(s), \zeta^k_g \rangle\right) \nabla_\theta \log \pi_{\theta_k}(a|s)\right]\right] - \nabla_\theta J_g(\theta_k)\nonumber\\
&= \mathbb{E}_{\theta_k}\left[\mathbb{E}_k\left[\left(g(s,a) - \eta^k_g + \langle \phi_g(s') - \phi_g(s), \zeta^k_g \rangle\right)\right] \nabla_\theta \log \pi_{\theta_k}(a|s)\right] - \nabla_\theta J_g(\theta_k)\nonumber\\
&= \mathbb{E}_{\theta_k}\left[\left(g(s,a) - \mathbb{E}_k[\eta^k_g] + \langle \phi_g(s') - \phi_g(s), \mathbb{E}_k[\zeta^k_g] \rangle\right) \nabla_\theta \log \pi_{\theta_k}(a|s)\right] - \nabla_\theta J_g(\theta_k).
\end{align}

Now, using the same decomposition as before, we add and subtract $\eta^*_g(\theta_k) = J_g(\theta_k)$ and $\zeta^*_g(\theta_k)$:
\begin{align}
&\mathbb{E}_{\theta_k}\left[\mathbb{E}_{k}\left[\hat{\nabla}_\theta J_g(\theta_k, \xi^k_g; z)\right]\right] - \nabla_\theta J_g(\theta_k) \nonumber\\
&= \mathbb{E}_{\theta_k}\left[\left(\eta^*_g(\theta_k) - \mathbb{E}_k[\eta^k_g] + \langle \phi_g(s') - \phi_g(s), \mathbb{E}_k[\zeta^k_g] - \zeta^*_g(\theta_k) \rangle\right) \nabla_\theta \log \pi_{\theta_k}(a|s)\right] \nonumber\\
&\quad + \mathbb{E}_{\theta_k}\left[\left(\langle \phi_g(s), \zeta^*_g(\theta_k) \rangle - V^{\pi_{\theta_k}}_g(s) + V^{\pi_{\theta_k}}_g(s') - \langle \phi_g(s'), \zeta^*_g(\theta_k) \rangle\right) \nabla_\theta \log \pi_{\theta_k}(a|s)\right] \nonumber\\
&\quad + \mathbb{E}_{\theta_k}\left[\left(g(s,a) - \eta^*_g(\theta_k) + V^{\pi_{\theta_k}}_g(s) - V^{\pi_{\theta_k}}_g(s')\right) \nabla_\theta \log \pi_{\theta_k}(a|s)\right] - \nabla_\theta J_g(\theta_k)
\end{align}

As before, the last term is 0 by Bellman's equation, and the norm squared of the second term is less than $4G_1^2 \epsilon_{\mathrm{app}}$ by Definition~\ref{def:critic_approx}. For the first term, we decompose further using the projection $\Pi$ onto $\ker(M_{\theta_k})^\perp$:
\begin{align}
&\mathbb{E}_{\theta_k}\left[\left(\eta^*_g(\theta_k) - \mathbb{E}_k[\eta^k_g] + \langle \phi_g(s') - \phi_g(s), \mathbb{E}_k[\zeta^k_g] - \zeta^*_g(\theta_k) \rangle\right) \nabla_\theta \log \pi_{\theta_k}(a|s)\right]\nonumber\\
&= \mathbb{E}_{\theta_k}\left[\left(\eta^*_g(\theta_k) - \mathbb{E}_k[\eta^k_g] + \langle \phi_g(s') - \phi_g(s), \Pi(\mathbb{E}_k[\zeta^k_g] - \zeta^*_g(\theta_k)) \rangle\right) \nabla_\theta \log \pi_{\theta_k}(a|s)\right]\nonumber\\
&\quad + \mathbb{E}_{\theta_k}\left[\langle \phi_g(s') - \phi_g(s), \Pi^\perp(\mathbb{E}_k[\zeta^k_g] - \zeta^*_g(\theta_k)) \rangle \nabla_\theta \log \pi_{\theta_k}(a|s)\right]\nonumber\\
&= \mathbb{E}_{\theta_k}\left[\left(\eta^*_g(\theta_k) - \mathbb{E}_k[\eta^k_g] + \langle \phi_g(s') - \phi_g(s), \Pi(\mathbb{E}_k[\zeta^k_g] - \zeta^*_g(\theta_k)) \rangle\right) \nabla_\theta \log \pi_{\theta_k}(a|s)\right],
\end{align}
where the last equality follows from Lemma 13 in \citet{ganesh2025regret}: for $s, s' \in \mathcal{S}^{\theta_k}_R$, we have $\langle \phi_g(s') - \phi_g(s), \Pi^\perp(\mathbb{E}_k[\zeta^k_g] - \zeta^*_g(\theta_k)) \rangle = 0$ since $\Pi^\perp(\mathbb{E}_k[\zeta^k_g] - \zeta^*_g(\theta_k)) \in \ker(M_{\theta_k}) = Z_{\theta_k}$.

Therefore:
\begin{align}
&\left\|\mathbb{E}_{\theta_k}\left[\left(\eta^*_g(\theta_k) - \mathbb{E}_k[\eta^k_g] + \langle \phi_g(s') - \phi_g(s), \Pi(\mathbb{E}_k[\zeta^k_g] - \zeta^*_g(\theta_k)) \rangle\right) \nabla_\theta \log \pi_{\theta_k}(a|s)\right]\right\|^2\nonumber\\
&\quad \leq G_1^2 \mathbb{E}_{\theta_k}\left[\left(\eta^*_g(\theta_k) - \mathbb{E}_k[\eta^k_g] + \langle \phi_g(s') - \phi_g(s), \Pi(\mathbb{E}_k[\zeta^k_g] - \zeta^*_g(\theta_k)) \rangle\right)^2\right]\nonumber\\
&\quad \leq G_1^2 \left(|\eta^*_g(\theta_k) - \mathbb{E}_k[\eta^k_g]|^2 + \mathbb{E}_{\theta_k}\left[\langle \phi_g(s') - \phi_g(s), \Pi(\mathbb{E}_k[\zeta^k_g] - \zeta^*_g(\theta_k)) \rangle^2\right]\right)\nonumber\\
&\quad \leq G_1^2 \left(\|\Pi(\mathbb{E}_k[\xi^k_g] - \xi^*_g(\theta_k))\|^2 + 4\|\phi_g\|^2_\infty \|\Pi(\mathbb{E}_k[\zeta^k_g] - \zeta^*_g(\theta_k))\|^2\right)\nonumber\\
&\quad \leq \mathcal{O}\left(G_1^2 \|\Pi(\mathbb{E}_k[\xi^k_g] - \xi^*_g(\theta_k))\|^2\right),
\end{align}
where we used $\|\phi_g\|_\infty \leq 1$ and the fact that $\xi^k_g = [\eta^k_g, (\zeta^k_g)^\top]^\top$.

Combining all terms:
\begin{equation}
\left\|\mathbb{E}_k\left[\mathbb{E}_{\theta_k}\left[\hat{\nabla}_\theta J_g(\theta_k, \xi^k_g; z)\right]\right] - \nabla_\theta J_g(\theta_k)\right\|^2 = \bar{\delta}_{k,g}^2 \leq \mathcal{O}\left(G_1^2 \|\Pi(\mathbb{E}_k[\xi^k_g] - \xi^*_g(\theta_k))\|^2 + G_1^2 \epsilon_{\mathrm{app}}\right).
\end{equation}

Now, applying Lemma~\ref{lemma: mlmc statements}, we obtain:
\begin{align}
\left\|\mathbb{E}\left[\hat{\nabla}_\theta J^{\text{MLMC}}_{g,kh}\right] - \nabla_\theta J_g(\theta_k)\right\|^2 &\leq \mathcal{O}\left(\frac{C^2_{\mathrm{tar}} G_1^2 \mathbb{E}_k\left[\|\Pi\xi^k_g\|^2\right]}{T_{\max}} + G_1^2 \|\Pi(\mathbb{E}_k[\xi^k_g] - \xi^*_g(\theta_k))\|^2 + G_1^2 \epsilon_{\mathrm{app}}\right),
\end{align}
where the first term arises from the MLMC bias due to finite trajectory length $T_{\max}$, with variance parameter $\mathbb{E}[\|\Pi \xi^k_g\|^2]$. This sharper bound, involving $\|\Pi(\mathbb{E}[\xi^k_g] - \xi^*_g(\theta_k))\|^2$ rather than $\mathbb{E}[\|\Pi(\xi^k_g - \xi^*_g(\theta_k))\|^2]$, is crucial for obtaining the optimal regret rate. The key insight is that by taking the expectation over $\xi^k_g$ before computing the norm, we isolate the bias of the critic estimate, which decays faster than its second moment.

\subsection{NPG Convergence Bounds}
With all conditions (C1)-(C9) verified, we can now apply Theorem~\ref{theorem: unichain stochastic linear recursion} to the NPG update. Setting the NPG step-size $\gamma_\omega = \frac{8 \log T}{\mu H}$, we obtain the following bounds for each $g \in \{r,c\}$.

\begin{theorem}[NPG Convergence for Unichain CMDPs]
Consider Algorithm~\ref{alg: pdnac_bi} under setting of Theorem~\ref{theorem: unichain_critic_final} with step size $\gamma_\omega = \frac{8 \log T}{\mu H}$. Then, conditioned on the event $\mathcal{E}_B$:
\begin{align}
    \mathbb{E}_k[\|\omega_g^k - \omega_{g,k}^*\|^2] 
    &\leq \mathcal{O}\left(\frac{G_1^2 C^2 c_\gamma^2}{\mu^2 T^2 \lambda^2} + \frac{G_1^6 C_{\mathrm{tar}}^2 c_\gamma^4 C^2}{\mu^4 \lambda^4} + \frac{G_1^2 \epsilon_{\mathrm{app}}}{\mu^2}\right), \\
    \|\mathbb{E}_k[\omega_g^k] - \omega_{g,k}^*\|^2 
    &\leq \tilde{\mathcal{O}}\left(\frac{G_1^2 c_\gamma^2 C_{\mathrm{tar}}^4}{\mu^2 \lambda^2 T^2} + \frac{G_1^{10} C_{\mathrm{tar}}^4 c_\gamma^2 C^2}{\mu^6 T_{\mathrm{max}} \lambda^4} + \frac{G_1^2 \epsilon_{\mathrm{app}}}{\mu^2}\right),
\end{align}
\end{theorem}

\begin{proof}
    To find the first and second moment bounds for \(\omega_g^k\), we use Theorem~\ref{theorem: unichain stochastic linear recursion},
    \begin{align}
        &\mathbb{E}_k \left[\|\omega_g^k - \omega_{g, k}^*\|^2\right] \\
        &\leq \mathcal{O} \left( \frac{1}{T^2} \|\omega_0 - \omega_{g, k}^*\|^2 + \frac{2 \log T}{\mu H} \mu^{-1}\left( \sigma_P^2 \mu^{-2} \Lambda_q^2 + \sigma_q^2\right) + \mu^{-2} \left(\delta_P^2 \mu^{-2} \Lambda_q^2 + \delta_q^2\right) \right) \nonumber \\
        &\leq \mathcal{O} \Bigg( \frac{1}{T^2} \|\omega_0 - \omega_{g, k}^*\|^2 \nonumber \\
        &\quad + \left( \frac{G_1^6 C_{\mathrm{tar}}^2 \mu^{-4} C^2}{H} + \frac{C^2_{\mathrm{tar}} \mu^{-2} G_1^2 \mathbb{E}_k\left[\|\Pi \xi^k_g\|^2\right]}{H} + \frac{G_1^2 \mathbb{E}_k\left[\|\Pi(\xi^k_g - \xi^*_g(\theta_k))\|^2\right] \mu^{-2}}{H} + \frac{G_1^2 \epsilon_{\mathrm{app}} \mu^{-2}}{H}\right) \nonumber \\
        &\quad + \left(\frac{G_1^6 C_{\mathrm{tar}}^2 \mu^{-4} C^2}{T_\mathrm{max}}  + \frac{C^2_{\mathrm{tar}} \mu^{-2} G_1^2 \mathbb{E}_k\left[\|\Pi \xi^k_g\|^2\right]}{T_{\max}} + G_1^2 \mu^{-2} \mathbb{E}_k\left[\|\Pi(\xi^k_g - \xi^*_g(\theta_k))\|^2\right] + G_1^2 \mu^{-2} \epsilon_{\mathrm{app}}\right) \Bigg) \\
        &\stackrel{(a)}{\leq} \mathcal{O} \Bigg( \frac{1}{T^2} \|\omega_0 - \omega_{g, k}^*\|^2 + \left(\frac{1}{T_\mathrm{max}} + \frac{1}{H}\right) \left(G_1^6 C_{\mathrm{tar}}^2 \mu^{-4} C^2 + C_\mathrm{tar}^2 \mu^{-2} G_1^2 \lambda^{-2} c_\gamma^2\right) \nonumber \\
        &\quad + \left(\mu^{-2} G_1^2\right) \left(\mathbb{E}_k\left[\|\Pi(\xi^k_g - \xi^*_g(\theta_k))\|^2\right] + \epsilon_{\mathrm{app}}\right) \Bigg)\nonumber \nonumber \\
        &\stackrel{(b)}{\leq} \mathcal{O} \Bigg( \frac{G_1^2 C^2}{\mu^2 T^2} + \left(\frac{1}{T_\mathrm{max}} + \frac{1}{H}\right) \frac{G_1^6 C_{\mathrm{tar}}^2 c_\gamma^2 C^2}{\mu^4 \lambda^2} + \frac{G_1^2}{\mu^{2}} \left(\mathbb{E}_k\left[\|\Pi(\xi^k_g - \xi^*_g(\theta_k))\|^2\right] + \epsilon_{\mathrm{app}}\right) \Bigg)
    \end{align}
    where \((a)\) uses 
    \begin{align*}
        \mathbb{E}_{k}\left[\|\Pi \xi_g^k\|^2\right] &\leq 2 \mathbb{E}_k \left[\|\Pi(\xi_g^k - \xi_g^*(\theta_k))\|^2\right] + 2 \mathbb{E}_k\left[\|\Pi \xi^*_g(\theta_k)\|^2\right] \\
        &\leq \mathcal{O}(\mathbb{E}_k\left[\|\Pi(\xi^k_g - \xi^*_g(\theta_k))\|^2\right] + \lambda^{-2} c_\gamma^2)
    \end{align*}
    and \((b)\) uses \(\omega_0 = 0\) and \(\|\omega_{g, k}^*\| = \|F(\theta_k)^{-1} J_g(\theta_k)\| = \mathcal{O}(\mu^{-1} CG_1)\). Similarly, we have
    \begin{align}
        \|\mathbb{E}_k[\omega_g^k] - \omega_{g,k}^*\|^2 &\leq \frac{G_1^2}{\mu^2} \tilde{\mathcal{O}} \left( \Pi \|\mathbb{E}_k[\xi_g^k] - \xi_g^*(\theta_k)\|^2 + \epsilon_{\mathrm{app}} \right) + \left( \frac{1}{T^2} + \frac{G_1^4 C_{\mathrm{tar}}^2}{\mu^2 T_{\text{max}}} \right) \|\omega_0 - \omega_{g,k}^*\|^2  \nonumber \\
        &\quad + \tilde{\mathcal{O}} \left( \frac{G_1^4 C_{\mathrm{tar}}^2}{\mu^2 T_{\mathrm{max}}} \left( \frac{G_1^2 C_{\mathrm{tar}}^2}{\mu^2}  \mathbb{E}_k \left[ \|\Pi(\xi_g^k - \xi_g^*(\theta_k))\|^2 \right] + \frac{G_1^6 C_{\mathrm{tar}}^2 c_\gamma^2 C^2}{\mu^4 \lambda^2} \right) \right) \\
        &\leq \frac{G_1^2}{\mu^2} \tilde{\mathcal{O}} \left( \Pi \|\mathbb{E}_k[\xi_g^k] - \xi_g^*(\theta_k)\|^2 + \epsilon_{\mathrm{app}} \right) + \tilde{\mathcal{O}} \left( \frac{G_1^2 C^2}{\mu^2 T^2} + \frac{G_1^6 C^2 C_{\mathrm{tar}}^2}{\mu^4 T_{\mathrm{max}}} \right) \nonumber \\
        &\quad + \tilde{\mathcal{O}} \left( \frac{G_1^6 C_{\mathrm{tar}}^4}{\mu^4 T_\mathrm{max}}  \mathbb{E}_k \left[ \|\Pi(\xi_g^k - \xi_g^*(\theta_k))\|^2 \right] + \frac{G_1^{10} C_{\mathrm{tar}}^4 c_\gamma^2 C^2}{\mu^6 T_{\mathrm{max}} \lambda^4} \right)
    \end{align}
    Finally, we can incorporate the results from Theorem~\ref{theorem: unichain_critic_final} to obtain
    \begin{align}
        \mathbb{E}_k \left[\|\omega_g^k - \omega_{g, k}^*\|^2\right] &\leq \tilde{\mathcal{O}} \Bigg( \frac{G_1^2 C^2}{\mu^2 T^2} + \frac{G_1^6 C_{\mathrm{tar}}^2 c_\gamma^2 C^2}{\mu^4 \lambda^2} + \frac{G_1^2}{\mu^{2}} \left(\frac{c_\gamma^2}{\lambda^2 T^2} + \lambda^{-4} c_\gamma^4 C_{\mathrm{tar}}^2 + \epsilon_{\mathrm{app}}\right) \Bigg) \\
        &\leq \tilde{\mathcal{O}} \Bigg( \frac{G_1^2 C^2 c_\gamma^2}{\mu^2 T^2 \lambda^2} + \frac{G_1^6 C_{\mathrm{tar}}^2 c_\gamma^4 C^2}{\mu^4 \lambda^4} + \frac{G_1^2}{\mu^{2}} \epsilon_{\mathrm{app}} \Bigg) \nonumber \\
        \|\mathbb{E}_k[\omega_g^k] - \omega_{g,k}^*\|^2 &\leq \frac{G_1^2}{\mu^2} \tilde{\mathcal{O}} \left( \frac{c_\gamma^2}{T^2 \lambda^2} + \frac{c_\gamma^4 C_\mathrm{tar}^2}{\lambda^4 T_\mathrm{max}} + \epsilon_{\mathrm{app}} \right) + \tilde{\mathcal{O}} \left( \frac{G_1^2 C^2}{\mu^2 T^2} + \frac{G_1^6 C^2 C_{\mathrm{tar}}^2}{\mu^4 T_{\mathrm{max}}} \right) \nonumber \\
        &\quad + \tilde{\mathcal{O}} \left( \frac{G_1^6 C_{\mathrm{tar}}^4}{\mu^4 T_\mathrm{max}} \left(\frac{c_\gamma^2}{\lambda^2 T^2} + \lambda^{-4} c_\gamma^4 C_{\mathrm{tar}}^2\right) + \frac{G_1^{10} C_{\mathrm{tar}}^4 c_\gamma^2 C^2}{\mu^6 T_{\mathrm{max}} \lambda^2} \right) \nonumber \\
        &\leq \tilde{\mathcal{O}}\left(\frac{G_1^2 c_\gamma^2 C_{\mathrm{tar}}^4}{\mu^2 \lambda^2 T^2} + \frac{G_1^{10} C_{\mathrm{tar}}^4 c_\gamma^2 C^2}{\mu^6 T_{\mathrm{max}} \lambda^4} + \frac{G_1^2}{\mu^2} \epsilon_{\mathrm{app}}\right)
    \end{align}
\end{proof}

%% file: Sections/Appendix/cmdp_analysis.tex
\section{Unichain CMDP Regret Analysis}
\label{sec: unichain cmdp regret analysis proofs}
Using Theorem~\ref{theorem: npg_convergence}, we can determine the first and second moment bounds for the full NPG direction \(\omega_k = \omega_r^k + \lambda_k \omega_c^k\)
\begin{align}
    \mathbb{E} \|\mathbb{E}_k \left[\omega_k\right] - \omega_k^*\| &\leq \left(1 + \frac{2}{\delta}\right) \tilde{\mathcal{O}} \left(\frac{G_1 C_{\mathrm{tar}}^2 c_\gamma}{\mu T \lambda} + \frac{G_1^5 C_{\mathrm{tar}}^2 c_\gamma C}{\mu^3 \lambda^2 \sqrt{T_\mathrm{max}}} +  \frac{G_1}{\mu} \sqrt{\epsilon_{\mathrm{app}}}\right) \label{eqn: omega first order bound}\\
    \mathbb{E} \left[\|\omega_k - \omega_k^*\|^2\right] &\leq \left(1 + \frac{4}{\delta^2}\right) \tilde{\mathcal{O}} \Bigg( \frac{G_1^2 C^2 c_\gamma^2}{\mu^2 T^2 \lambda^2} + \frac{G_1^6 C_{\mathrm{tar}}^2 c_\gamma^4 C^2}{\mu^4 \lambda^4} + \frac{G_1^2}{\mu^{2}} \epsilon_{\mathrm{app}} \Bigg) \label{eqn: omega second order bound}
\end{align}
To achieve regret bounds for unichain CMDPs, we leverage Lemma 4.6 from \citet{xu2025global} to compare the convergence of the Lagrange function with the critic and NPG to yield
\begin{align*}
    &\mathbb{E} \sum_{k=0}^{K-1} \left( \mathcal{L}(\pi^*, \lambda_k) - \mathcal{L}(\theta_k, \lambda_k) \right) \\
    &\leq K\sqrt{\epsilon_{\mathrm{bias}}} + G_1 \sum_{k=0}^{K-1} \mathbb{E} \left\| \left( \mathbb{E}_k [\omega_k] - \omega_k^* \right) \right\| + \frac{\alpha G_2}{2} \sum_{k=0}^{K-1} \mathbb{E}_k \|\omega_k\|^2\\
    &\quad + \frac{1}{\alpha } \mathbb{E}_{s \sim d^{\pi^*}} [ KL(\pi^*(\cdot|s) \| \pi_{\theta_0}(\cdot|s))] \\
    &\leq K \sqrt{\epsilon_{\mathrm{bias}}} + G_1 \sum_{k=0}^{K-1} \mathbb{E} \left\| \left( \mathbb{E}_k [\omega_k] - \omega_k^* \right) \right\| + \frac{\alpha G_2}{2} \sum_{k=0}^{K-1} \mathbb{E}_k \|\omega_k - \omega_k^*\|^2\\
    &\quad  + \frac{\alpha G_2}{2 \mu^2 } \sum_{k=0}^{K-1} \mathbb{E}_k \|\nabla \mathcal{L}(\theta_k, \lambda_k)\|^2 + \frac{1}{\alpha } \mathbb{E}_{s \sim d^{\pi^*}} [ KL(\pi^*(\cdot|s) \| \pi_{\theta_0}(\cdot|s))] \\
    &\stackrel{(a)}{\leq} K\sqrt{\epsilon_{\mathrm{bias}}} + G_1 \sum_{k=0}^{K-1} \left(1 + \frac{2}{\delta}\right) \tilde{\mathcal{O}} \left( \frac{G_1^5 C_{\mathrm{tar}}^2 c_\gamma C}{\mu^3 \lambda^2 \sqrt{T}} +  \frac{G_1}{\mu} \sqrt{\epsilon_{\mathrm{app}}}\right) \\
    &\quad + \frac{\alpha G_2}{2} \sum_{k=0}^{K-1} \left(1 + \frac{4}{\delta^2}\right) \tilde{\mathcal{O}} \Bigg( \frac{G_1^6 C_{\mathrm{tar}}^2 c_\gamma^4 C^2}{\mu^4 \lambda^4} + \frac{G_1^2}{\mu^{2}} \epsilon_{\mathrm{app}} \Bigg)\\
    &\quad  + \frac{\alpha G_2}{2 \mu^2 } \sum_{k=0}^{K-1} \mathbb{E} \|\nabla \mathcal{L}(\theta_k, \lambda_k)\|^2 + \frac{1}{\alpha } \mathbb{E}_{s \sim d^{\pi^*}} [ KL(\pi^*(\cdot|s) \| \pi_{\theta_0}(\cdot|s))] \\
\end{align*}
where \((a)\) uses the bounds in equations~\eqref{eqn: omega first order bound} and~\eqref{eqn: omega second order bound} with \(T_{\mathrm{max}} = \mathcal{O}(T)\). Then by Lemma~\ref{lemma: pg and lagrangian bounds} and the definition of the Lagrange function, we have
\begin{align}
    \label{eqn: lagrange final bound}
    &\mathbb{E} \sum_{k=0}^{K-1} \left(J_r^{\pi^*} - J_r(\theta_k)\right) + \mathbb{E} \sum_{k=0}^{K-1} \lambda_k \left(J_c^{\pi^*} - J_c(\theta_k)\right) \nonumber \\
    &\leq K \sqrt{\epsilon_{\mathrm{bias}}} + K\left(1 + \frac{2}{\delta}\right) \tilde{\mathcal{O}} \left( \frac{G_1^6 C_{\mathrm{tar}}^2 c_\gamma C}{\mu^3 \lambda^2 \sqrt{T}} +  \frac{G_1^2}{\mu} \sqrt{\epsilon_{\mathrm{app}}}\right) \nonumber \\
    &\quad + \alpha G_2 K \left(1 + \frac{4}{\delta^2}\right) \tilde{\mathcal{O}} \Bigg( \frac{G_1^6 C_{\mathrm{tar}}^2 c_\gamma^4 C^2}{\mu^4 \lambda^4} + \frac{G_1^2}{\mu^{2}} \epsilon_{\mathrm{app}} + \frac{G_1^2 C^2}{\mu^2}\Bigg) \nonumber \\
    &\quad + \frac{1}{\alpha} \mathbb{E}_{s \sim d^{\pi^*}} [ KL(\pi^*(\cdot|s) \| \pi_{\theta_0}(\cdot|s))]
\end{align}
From \citet{xu2025global}, we have
\begin{align*}
    \mathbb{E} \sum_{k=0}^{K-1} \lambda_k \left(J_c^{\pi^*} - J_c(\theta_k)\right) &\geq - \sum_{k=0}^{K-1} \mathbb{E} \left[\lambda_k \left(J_c(\theta_k) - \eta_c^k\right)\right] - \beta \sum_{k=0}^{K-1} \mathbb{E} \left[\eta_c^k\right]^2 \\
    &\geq - \sum_{k=0}^{K-1} \mathbb{E} \left[\lambda_k \left|J_c(\theta_k) - \eta_c^k\right|\right] - \beta K \\
    &\geq -\tilde{\mathcal{O}} \left(\frac{2K}{\delta}\sqrt{\frac{C^2_\mathrm{tar} c_\gamma^4}{T \lambda^4}} + \beta K\right) \\
    &\geq -\tilde{\mathcal{O}} \left(\frac{KC_\mathrm{tar} c_\gamma^2}{\delta \lambda^2 \sqrt{T}} + \beta K\right) \\
\end{align*}
Incorporating this result into Equation~\eqref{eqn: lagrange final bound} yields
\begin{align}
    \mathbb{E} \sum_{k=0}^{K-1} \left(J_r^{\pi^*} - J_r(\theta_k)\right) &\leq  K\sqrt{\epsilon_{\mathrm{bias}}} + K\left(1 + \frac{2}{\delta}\right) \tilde{\mathcal{O}} \left( \frac{G_1^6 C_{\mathrm{tar}}^2 c_\gamma C}{\mu^3 \lambda^2 \sqrt{T}} +  \frac{G_1^2}{\mu} \sqrt{\epsilon_{\mathrm{app}}}\right) \nonumber \\
    &\quad + \alpha G_2 K \left(1 + \frac{4}{\delta^2}\right) \tilde{\mathcal{O}} \Bigg( \frac{G_1^6 C_{\mathrm{tar}}^2 c_\gamma^4 C^2}{\mu^4 \lambda^4} + \frac{G_1^2}{\mu^{2}} \epsilon_{\mathrm{app}} + \frac{G_1^2 C^2}{\mu^2}\Bigg) \nonumber \\
    &\quad + \frac{1}{\alpha} \mathbb{E}_{s \sim d^{\pi^*}} [ KL(\pi^*(\cdot|s) \| \pi_{\theta_0}(\cdot|s))] + \tilde{\mathcal{O}} \left(\frac{K C_\mathrm{tar} c_\gamma^2}{\delta \lambda^2 \sqrt{T}} + \beta K\right) \\
    &\leq \tilde{\mathcal{O}} \left(K\sqrt{\epsilon_\mathrm{bias}} + \frac{K C_\mathrm{tar}^2 C}{\sqrt{T}} + K\sqrt{\epsilon_\mathrm{app}} + \alpha K + \beta K + \frac{1}{\alpha}\right)
\end{align}
Additionally, from \citet{xu2025global}, we have
\begin{align}
    \frac{1}{K} \sum_{k=0}^{K-1} \mathbb{E} \left[J_c(\theta_k) \left(\lambda_k - \frac{2}{\delta}\right)\right] \leq \frac{2}{\delta^2\beta K} + \frac{\beta }{2}
\end{align}
Then using Lemma G.6 in \citet{xu2025global} yields
\begin{align}
    &\mathbb{E} \sum_{k=0}^{K-1} \left(- J_c(\theta_k)\right) \nonumber \\
    &\leq K\left(1 + \frac{\delta}{2}\right) \tilde{\mathcal{O}} \left( \frac{G_1^6 C_{\mathrm{tar}}^2 c_\gamma C}{\mu^3 \lambda^2 \sqrt{T}} +  \frac{G_1^2}{\mu} \sqrt{\epsilon_{\mathrm{app}}}\right) + \frac{\delta}{2\alpha} \mathbb{E}_{s \sim d^{\pi^*}} [ KL(\pi^*(\cdot|s) \| \pi_{\theta_0}(\cdot|s))] \nonumber \\
    &\quad + \alpha G_2 K \left(\frac{2}{\delta} + \frac{\delta}{2}\right) \tilde{\mathcal{O}} \Bigg( \frac{G_1^6 C_{\mathrm{tar}}^2 c_\gamma^4 C^2}{\mu^4 \lambda^4} + \frac{G_1^2}{\mu^{2}} \epsilon_{\mathrm{app}} + \frac{G_1^2 C^2}{\mu^2}\Bigg) + \delta K \sqrt{\epsilon_\mathrm{bias}} + \frac{1}{\delta \beta} + \frac{\beta K}{2} \\
    &\leq \tilde{\mathcal{O}} \left(K\sqrt{\epsilon_\mathrm{bias}} + \frac{K C_\mathrm{tar}^2 C}{\sqrt{T}} + K\sqrt{\epsilon_\mathrm{app}} + \alpha K + \beta K + \frac{1}{\alpha} + \frac{1}{\beta}\right)
\end{align}

\begin{theorem}
    Under the setup of PDNAC-BI with the parameter choices \(H = \Theta(\log T),
    K = \Theta\left(\frac{T}{\log T}\right),
    \alpha = \Theta\left(\frac{1}{\sqrt{T}}\right), \beta = \Theta \left(\frac{1}{\sqrt{T}}\right)\),    the expected regret and cumulative constraint violation satisfy
    \begin{align*}
        \mathbb{E}[\mathrm{Reg}_T] &= \tilde{\mathcal{O}} \left(T(\sqrt{\epsilon_{\mathrm{bias}}}+\sqrt{\epsilon_{\mathrm{app}}}) + C_{\mathrm{tar}}^2 C\,\sqrt{T}\right),\\
        \mathbb{E}\left[\sum_{k=0}^{K-1}(-J_c(\theta_k))\right] &= \tilde{\mathcal{O}} \left(T(\sqrt{\epsilon_{\mathrm{bias}}}+\sqrt{\epsilon_{\mathrm{app}}}) + C_{\mathrm{tar}}^2 C\,\sqrt{T}\right),
    \end{align*}
    where $\epsilon_{\mathrm{bias}}$ and $\epsilon_{\mathrm{app}}$ are the actor and critic approximation errors, and $C=C_{\mathrm{hit}}+C_{\mathrm{tar}}$.
\end{theorem}

\begin{proof}
    We first decompose the regret into two parts, following standard analysis in prior work \citep{ganesh2025regret, bai2024regret}
    \begin{align}
        \mathbb{E} \left[\mathrm{Reg}_T \mid \mathcal{E}_B\right] &\leq H \mathbb{E} \sum_{k=0}^{K-1} \left(J_r^{\pi^*} - J_r(\theta_k)\right) + \mathbb{E}\sum_{k=0}^K \sum_{t \in \mathcal{I}_k} \left(J_r(\theta_k) - r(s_t, a_t) \right) \\
        &\leq H \tilde{\mathcal{O}} \left(K\sqrt{\epsilon_\mathrm{bias}} + \frac{K C_\mathrm{tar}^2 C}{\sqrt{T}} + K\sqrt{\epsilon_\mathrm{app}} + \alpha K + \beta K + \frac{1}{\alpha}\right) \nonumber \\
        &\quad + \mathbb{E} \sum_{k=0}^{K-1} \sum_{t \in \mathcal{I}_{k}} \left(J_r(\theta_k) - r(s_t, a_t) \right)
    \end{align}
    Many prior works have bounded the section term via a constant bound on the value function for the inner sum. However, this results in a regret bound of \(O(CK)\), which is not sublinear if we need to set \(K = \mathcal{O}(T / \log T)\). Instead, since \(\nabla_\theta V^{\pi_\theta}(s)\) is bounded by \((1+2C)G_1\), then \(V^{\pi_\theta}(s)\) is \((1+2C)G_1\)-Lipschitz. Thus,
    \begin{align*}
        &\mathbb{E}\sum_{k=0}^{K-1} \sum_{t \in \mathcal{I}_{k}} \left(J_r(\theta_k) - r(s_t, a_t) \right) \\
        &= \mathbb{E} \left[\sum_{k=0}^{K=1} V^{\pi_{\theta_{k+1}}}(s_{kH}) - V^{\pi_{\theta_{k}}}(s_{kH})\right] + \mathbb{E} \left[V^{\pi_{\theta_K}}(s_T)\right] - \mathbb{E}\left[V^{\pi_{\theta_0}}(s_0)\right] \\
        &\leq (1+2C)G_1\sum_{k=0}^{K-1} \mathbb{E} \|\theta_{k+1} - \theta_k\| + 2C \\
        &\leq (1+2C)G_1 \alpha \sum_{k=0}^{K-1} \mathbb{E} \|\omega_k\| + 2C \\
        &\leq \mathcal{O}\left(CG_1 \alpha K\right)
    \end{align*}
    Combining this with the decomposition in Equation~\eqref{eqn: regret decomposition} yields
    \begin{align}
        \mathbb{E} \left[\mathrm{Reg}_T\right] &\leq \mathbb{E}[\mathrm{Reg}_T \mid \mathcal{E}_B] + \mathcal{O}\left(\frac{T^2 2^{\left\lfloor-\frac{B}{2C_{\mathrm{hit}}})\right\rfloor}}{\log T}\right) \\
        &\leq H \tilde{\mathcal{O}} \left(K\sqrt{\epsilon_\mathrm{bias}} + \frac{K C_\mathrm{tar}^2 C}{\sqrt{T}} + K\sqrt{\epsilon_\mathrm{app}} + \alpha K + \beta K + \frac{1}{\alpha}\right) + \mathcal{O}\left(CG_1 \alpha K\right) + o(1) \\
        &\leq H \tilde{\mathcal{O}} \left(K\sqrt{\epsilon_\mathrm{bias}} + \frac{K C_\mathrm{tar}^2 C}{\sqrt{T}} + K\sqrt{\epsilon_\mathrm{app}} + \alpha C K + \beta K + \frac{1}{\alpha}\right) + o(1) \\
        &\leq \tilde{\mathcal{O}} \left(T\sqrt{\epsilon_\mathrm{bias}} + \frac{T C_\mathrm{tar}^2 C}{\sqrt{T}} + T\sqrt{\epsilon_\mathrm{app}} + \alpha C T + \beta T + \frac{1}{\alpha}\right) \\
        &\leq \tilde{\mathcal{O}} \left(T (\sqrt{\epsilon_\mathrm{bias}} + \sqrt{\epsilon_\mathrm{app}}) + \sqrt{T} C_\mathrm{tar}^2 C\right)
    \end{align}
    Additionally, our cumulative constraint violation is also sublinear with
    \begin{align}
        \mathbb{E} \sum_{k=0}^{K-1} \left(- J_c(\theta_k)\right) &\leq \tilde{\mathcal{O}} \left(K\sqrt{\epsilon_\mathrm{bias}} + \frac{K C_\mathrm{tar}^2 C}{\sqrt{T}} + K\sqrt{\epsilon_\mathrm{app}} + \alpha K + \beta K + \frac{1}{\alpha} + \frac{1}{\beta}\right) \\
        &\leq \tilde{\mathcal{O}} \left(T (\sqrt{\epsilon_\mathrm{bias}} + \sqrt{\epsilon_\mathrm{app}}) + \sqrt{T} C_\mathrm{tar}^2 C\right)
    \end{align}
\end{proof}

%% file: Sections/Appendix/auxilliary_lemmas.tex
\section{Auxiliary Lemmas}
\begin{lemma}
    \label{lemma: value, q, advantange function bounds}
    For any unichain MDP with hitting time constant $C_{\mathrm{hit}}$ and target mixing constant $C_{\mathrm{tar}}$, let $C = C_{\mathrm{hit}} + C_{\mathrm{tar}}$. The following holds $\forall(s, a) \in \mathcal{S} \times \mathcal{A}$, any policy $\pi$ and $\forall g \in \{r, c\}$.
    \begin{equation*}
        (a) |V_g^\pi(s)| \leq 2C, \quad (b) |Q_g^\pi(s, a)| \leq 1 + 2C, \quad (c) |A_g^\pi(s, a)| \leq 1 + 4C
    \end{equation*}
\end{lemma}

\begin{lemma}
    \label{lemma: pg and lagrangian bounds}
    For any unichain MDP with constants defined in Lemma G.1, the following holds $\forall(s, a) \in \mathcal{S} \times \mathcal{A}$ for any policy $\pi_{\theta_k}$ with $\theta_k$ satisfying Assumption 4.4, the dual parameter $\lambda_k \in [0, 2/\delta]$, and $g \in \{r, c\}$.
    \begin{equation*}
         \|\nabla_\theta J_g(\theta_k)\| \leq (1 + 2C)G_1, \quad  \|\nabla_\theta \mathcal{L}(\theta_k, \lambda_k)\| \leq (1 + 2C)\left(1 + \frac{2}{\delta}\right)G_1
    \end{equation*}
\end{lemma}
\begin{proof}
     Using the policy gradient theorem for unichain MDPs, which expresses the gradient in terms of the stationary distribution $d^{\pi_{\theta_k}}$, we have:
\begin{equation}
    \nabla_\theta J_g(\theta_k) = \mathbb{E}_{(s,a) \sim \nu^{\pi_{\theta_k}}} \left[ Q_g^{\pi_{\theta_k}}(s, a) \nabla_\theta \log \pi_{\theta_k}(a|s) \right]
\end{equation}
Applying the bound from Lemma~\ref{lemma: value, q, advantange function bounds} and the score function bound from Assumption~\ref{assumption: score bounds}:
\begin{equation}
    \|\nabla_\theta J_g(\theta_k)\| = \left\| \mathbb{E}_{(s,a) \sim \nu^{\pi_{\theta_k}}} \left[ Q_g^{\pi_{\theta_k}}(s, a) \nabla_\theta \log \pi_{\theta_k}(a|s) \right] \right\| \leq (1 + 2C)G_1
\end{equation}
Combining the above result with the definition of the Lagrange function $\mathcal{L}(\theta, \lambda) = J_r(\theta) + \lambda J_c(\theta)$ and the dual bound $\lambda_k \leq \frac{2}{\delta}$ derived from the Slater condition:
\begin{align}
    \|\nabla_\theta \mathcal{L}(\theta_k, \lambda_k)\| &= \|\nabla_\theta J_r(\theta_k) + \lambda_k \nabla_\theta J_c(\theta_k)\| \nonumber \\
    &\leq \|\nabla_\theta J_r(\theta_k)\| + \lambda_k \|\nabla_\theta J_c(\theta_k)\| \nonumber \\
    &\leq (1 + 2C)G_1 + \frac{2}{\delta}(1 + 2C)G_1 = (1 + 2C)\left(1 + \frac{2}{\delta}\right)G_1 
\end{align}
\end{proof}

\begin{lemma}
For any fixed state $s$,
\begin{align*}
    \|\nabla_{\theta} V_g^{\pi_\theta}(s)\|  \le (1+2C)G_1.
\end{align*}
\end{lemma}

\begin{proof}
Using the score function identity and the fact that $\pi_\theta(\cdot\mid s)$ is a probability distribution,
\begin{align*}
    \|\nabla_\theta V_g^{\pi_\theta}(s)\| = \left\|
    \sum_{a}\pi_\theta(a\mid s) Q_g^{\pi_\theta}(s,a)\,\nabla_\theta \log \pi_\theta(a \mid s) \right\|
    &\leq \sum_{a}\pi_\theta(a\mid s) \big| Q_g^{\pi_\theta}(s,a)\big|\,
    \big\|\nabla_\theta\log\pi_\theta(a\mid s)\big\| \\
    &\leq (1+2C)G_1
\end{align*}
where the first inequality is the triangle inequality and the second uses the bounds
$|Q_g^{\pi_\theta}(s,a)| \leq (1+2C)$ and $\|\nabla_\theta\log\pi_\theta(a\mid s)\|\le G_1$.
\end{proof}

\begin{lemma}[Regret decomposition under burn-in]
Let $\mathcal{E}_B := \bigcap_{k=0}^{K-1} \{ T^{\theta_k} \le B \}$ denote the event that, in every epoch, the burn-in of length $B$ under policy $\pi_{\theta_k}$ reaches the recurrent class. Then the expected regret admits the decomposition
\begin{align}
\mathbb{E}[\mathrm{Reg}_T] \le \mathbb{E}[\mathrm{Reg}_T \mid \mathcal{E}_B] + \mathcal{O}\!\left( \frac{T^2 2^{-B/(2C_{\mathrm{hit}})}}{\log T} \right).
\end{align}
\end{lemma}

By the unichain hitting-time bound (Lemma 10, \citet{ganesh2025regret}),
\begin{align}
\Pr\left( T^{\theta_k} > B \right) \leq 2^{-B/(2C_{\mathrm{hit}})}.
\end{align}
Since $\mathcal{S}^{\theta_k}_R$ is closed, conditioning on $T^{\theta_k} \le B$ ensures that all subsequent states remain in the recurrent class. Hence, all samples for critic and NPG updates lie entirely within $\mathcal{S}^{\theta_k}_R$, guaranteeing the required kernel inclusion. Define
\begin{align}
    \mathcal{E}_B := \bigcap_{k=0}^{K-1} \{ T^{\theta_k} \le B \}.
\end{align}
Since instantaneous regret is at most $1$, applying a union bound over epochs gives
\begin{align}
\mathbb{E}[\mathrm{Reg}_T] &\le \mathbb{E}[\mathrm{Reg}_T \mid \mathcal{E}_B] + T K 2^{-B/(2C_{\mathrm{hit}})} \\
&= \mathbb{E}[\mathrm{Reg}_T \mid \mathcal{E}_B] + \mathcal{O}\left(\frac{T^2 2^{-B/(2C_{\mathrm{hit}})}}{\log T}\right).
\label{eqn: regret decomposition}
\end{align}

%% file: Sections/Appendix/limitations.tex
\section{Limitations and Future Work.}
Despite relaxing several standard assumptions in constrained RL, several challenges remain. First, our analysis relies on linear function approximation for the critic. Extending the results to non-linear approximators such as neural networks requires controlling the stability of the critic matrix $\mathbf{A}_g(\theta)$ and addressing the non-convexity of the representation learning problem. Second, while we adopt the unichain assumption, which to the best of our knowledge is the weakest structural condition under which policy-gradient-based methods admit rigorous analysis, it still excludes some classes of weakly communicating MDPs. Extending policy-gradient methods to  multi-chain settings requires algorithms that can adaptively identify and exploit optimal recurrent classes without inducing instability in the gradient estimates. However, we expect this to be difficult since the weakly communicating setting has only been approached for value-iteration based methods. Finally, although we achieve $\tilde{\mathcal{O}}(\sqrt{T})$ cumulative constraint violation, achieving zero-constraint violation during learning would require new mechanisms that enforce feasibility without relying on prior knowledge of mixing or hitting-time properties.